\documentclass[letterpaper,onecolumn,11pt,accepted=2025-06-26]{quantumarticle}
\pdfoutput=1

\usepackage{amsmath,amssymb,amsthm}
\usepackage{algorithm2e}

\usepackage[numbers,sort&compress]{natbib}
\usepackage[hidelinks]{hyperref}
\hypersetup{colorlinks,
            linkcolor=blue,
            citecolor=blue,
            urlcolor=blue,
            linktocpage,
            plainpages=false,
            breaklinks=true}
\usepackage{url}
\usepackage{doi}

\usepackage[compact]{titlesec}
\usepackage{textcomp}
\usepackage{tikz}
\usetikzlibrary{patterns}
\usepackage{soul}
\usepackage{relsize}
\usepackage{graphicx}
\usepackage{subfigure}
\usepackage{enumitem}
\usepackage{braket}
\usepackage{colortbl}
\usepackage{xcolor}
\usepackage{multicol, multirow}
\usepackage[font=small,belowskip=-10pt]{caption}
\usepackage{dsfont}
\usepackage[utf8]{inputenc}
\usepackage[english]{babel}
\usepackage[T1]{fontenc}

\newcommand{\BEAS}{\begin{eqnarray*}}
\newcommand{\EEAS}{\end{eqnarray*}}
\newcommand{\BEA}{\begin{eqnarray}}
\newcommand{\EEA}{\end{eqnarray}}
\newcommand{\BEQ}{\begin{equation}}
\newcommand{\EEQ}{\end{equation}}
\newcommand{\BIT}{\begin{itemize}}
\newcommand{\EIT}{\end{itemize}}
\newcommand{\BPF}{\begin{proof}}
\newcommand{\EPF}{\end{proof}}

\newcommand{\reals}{{\mbox{\textbf{R}}}}

\newcommand{\Prob}{\mathds{P}}
\newcommand{\Expect}{\mathds{E}}
\newcommand{\argmax}{\mathop{\rm argmax}}

\newcommand{\indf}{\mathbf{1}}

\def\cH{{\mathcal{H}}}
\def\cY{{\mathcal{Y}}}
\def\cX{{\mathcal{X}}}
\def\cN{{\mathcal{N}}}
\def\cC{{\mathcal{C}}}

\def\cA{{\mathcal{A}}}
\def\cL{{\mathcal{L}}}
\def\cR{{\mathcal{R}}}

\def\cV{{\mathcal{V}}}

\newcommand{\bitem}{\begin{itemize}}
\newcommand{\eitem}{\end{itemize}}
\newcommand{\beq}{\begin{equation}}
\newcommand{\eeq}{\end{equation}}
\newcommand{\bea}{\begin{eqnarray}}
\newcommand{\eea}{\end{eqnarray}}
\newcommand{\bear}{\begin{eqnarray*}}
\newcommand{\eear}{\end{eqnarray*}}
\newcommand{\efig}{\end{figure}}

\newcommand{\nDim}{\text{Ndim}}

\newcommand{\vcDim}{\text{VCdim}}

\newcommand{\lDim}{\text{Ldim}}
\newcommand{\lDimH}{\lDim(\cH)}
\newcommand{\mclDim}{\text{mcLdim}}
\newcommand{\mclDimH}{\mclDim(\cH)}

\newcommand{\lossf}{\mathcal{L}}
\newcommand{\regretf}{\mathcal{R}}
\newcommand{\vech}{\mathbf{h}}
\newcommand{\vecD}{\mathbf{D}}
\newcommand{\vecx}{\mathbf{x}}
\newcommand{\vecz}{\mathbf{z}}

\newcommand{\bigO}{\mathcal{O}}
\newcommand{\cF}{\mathcal{F}}

\newtheorem{theorem}{Theorem}[section]
\newtheorem{lemma}[theorem]{Lemma}

\newtheorem{definition}[theorem]{Definition}

\newtheorem{remark}{Remark}

\begin{document}

\title{Quantum Learning Theory Beyond Batch Binary Classification}

\author{Preetham Mohan}
\affiliation{Department of Mathematics, University of Michigan}
\orcid{0000-0003-3401-2985}
\email{preetham@umich.edu}

\author{Ambuj Tewari}
\affiliation{Department of Statistics, University of Michigan}
\affiliation{Department of Electrical Engineering and Computer Science, University of Michigan}
\orcid{0000-0001-6969-7844}
\email{tewaria@umich.edu}

\maketitle

\begin{abstract}
Arunachalam and de Wolf \cite{arunachalam2018optimal} showed that the sample complexity of quantum batch learning of boolean functions, in the realizable and agnostic settings, has the \textit{same form and order} as the corresponding classical sample complexities. In this paper, we extend this, ostensibly surprising, message to batch multiclass learning, online boolean learning, and online multiclass learning.  For our online learning results, we first consider an adaptive adversary variant of the classical model of Dawid and Tewari \cite{dawid2022learnability}. Then, we introduce the first (to the best of our knowledge) model of online learning with quantum examples.
\end{abstract}

\section{Introduction}\label{sec:intro}

Bshouty and Jackson \cite{bshouty1995learning} provided a quantum extension of the PAC setting  by formalizing what it means to learn from quantum examples. This inspired a line of work \cite{servedio2004equivalences, atici2005improved, zhang2010improved}, culminating in the work of Arunachalam and de Wolf \cite{arunachalam2018optimal}, that has provided {\em sample complexity} bounds for quantum batch learning of boolean functions. Arunachalam and de Wolf \cite{arunachalam2018optimal} helped crystallize the following message:
\begin{enumerate}[noitemsep]
    \item No new combinatorial dimension is needed to characterize quantum batch learnability of boolean functions--the VC dimension continues to do so.
    \item There is \textit{at most a constant} sample complexity advantage for quantum batch learning of boolean functions, in both the realizable and agnostic settings, as compared to the corresponding classical sample complexities.
\end{enumerate}
In this paper, we show that this message continues to hold in three other learning settings: batch learning of multiclass functions, online learning of boolean functions, and online learning of multiclass functions.

Our motivation for considering quantum batch learning of multiclass functions is an open question posed in \cite{arunachalam2018optimal} which asks ``what is the quantum sample complexity for learning concepts whose range is $[k]$ rather than $\{0,1\}$, for some $k > 2$?'' We resolve this question for $2 < k < \infty$ (see Section \ref{sec:quantum-batch}). In classical multiclass batch learning, an approach to establish the lower and upper sample complexity bounds \cite{daniely2015multiclass} is to proceed via a reduction to the binary case, with an appeal to the definition of Natarajan dimension. While classically straightforward, extending such a proof approach to establish sample complexity bounds for quantum multiclass batch learning involves manipulating quantum examples, which has to be done with utmost care (see Section \ref{sec:quantum-PAC-k-lb}).

Unlike the batch setting, quantum online learning of classical functions, to the best of our knowledge, has no predefined model. One possible explanation is that we need, as an intermediary, a new classical online learning model (see Sections \ref{model:classical-adversary-provides-a-distribution}, \ref{model:classical-adversary-provides-a-distribution-agnostic}) where, at each round, the adversary provides a {\em distribution} over the example (input-label) space instead of a single example. With this new classical model, and the definition of a quantum example, a model for online learning in the quantum setting arises as a natural extension (see Figure \ref{figure:generalization-classical-quantum}).

\begin{figure}[ht]
\centering
\begin{tikzpicture}
\node[draw] (A1) at (-4,-1.5) {Classical Batch Learning};
\node[draw] (B1) at (4,-1.5) {Classical Online Learning};
\node[draw] (A2) at (-4,0) {Definition of a quantum example};
\node[draw, align=center, text width=2.3in] (B2) at (4,0) {Adversary-provides-a-distribution Model (Sections \ref{model:classical-adversary-provides-a-distribution}, \ref{model:classical-adversary-provides-a-distribution-agnostic})};
\node[draw] (A3) at (-4,1.5) {Quantum Batch Learning};
\node[draw] (B3) at (4,1.5) {Quantum Online Learning};
\draw[->]
  (A1) edge (A2) (A2) edge (A3) (B1) edge (B2);
\draw[-] (B2) -- (A2);
\draw[-] (0,0) -- (0,1.5);
\draw[->] (0,1.5) -- (B3);
\end{tikzpicture}
\caption{Mapping of the tools necessary for generalizations of learning paradigms from classical to quantum.} \label{figure:generalization-classical-quantum}
\end{figure}
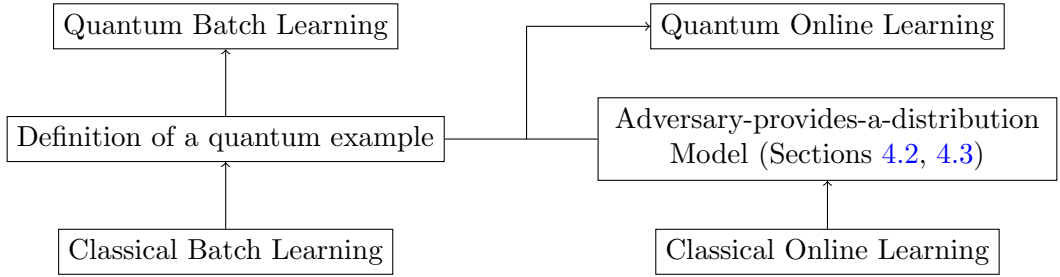

\subsection{Our Contributions}

In Tables \ref{table:results-in-batch-LT} and \ref{table:results-in-online-LT}, we provide a concise overview of existing results and highlight our contributions in batch and online learning for binary and multiclass classification across realizable and agnostic settings in both classical and quantum paradigms. Our contributions include:
\begin{itemize}[noitemsep]
    \item establishing lower and upper sample complexity bounds for quantum batch multiclass classification in the realizable and agnostic settings,
    \item proposing a new classical online learning model, which is an adaptive adversary variant of an existing classical online learning model~\cite{dawid2022learnability},
    \item proposing a quantum online learning model, as a natural generalization of our proposed classical online learning model,
    \item establishing tight expected regret bounds for quantum online binary classification in the realizable and agnostic settings,
    \item establishing a tight expected regret bound for quantum online multiclass classification in the realizable setting, and
    \item establishing lower and upper expected regret bounds for quantum online multiclass classification in the agnostic setting.
\end{itemize}

\paragraph{Notes on Tables \ref{table:results-in-batch-LT} and \ref{table:results-in-online-LT}}

\begin{itemize}[noitemsep]
    \item In all cases, we state known results that exhibit the tightest dependence on the combinatorial parameters that characterize learning in the respective settings.
    \item In the batch multiclass case, we work with Natarajan dimension, instead of DS dimension which was shown to characterize classical batch multiclass learning (including the $k \to \infty$ case) recently \cite{brukhim2022characterization}. We defer the resolution of the quantum sample complexity in the $k \to \infty$ case to future work.
    \item In the batch multiclass realizable case, there exists an upper bound with a tighter dependence on $\epsilon$ (but looser on $\nDim(\cH)$) for both classical and quantum cases (see Section \ref{sec:classical-quantum-multiclass-PAC-ub-proof}).
    \item In the (canonical) classical online multiclass agnostic case, the $\tilde{\bigO}(\sqrt{\mclDimH T})$ bound hides a $\sqrt{\log\Big(\frac{T}{\mclDimH}\Big)}$ factor \cite{hanneke2023multiclass}.
    \item The definition of loss/regret differs between the canonical classical adversary-provides-an-input model (in Section \ref{model:classical-adversary-provides-an-input}) and both the classical adversary-provides-a-distribution model (in Sections \ref{model:classical-adversary-provides-a-distribution}, \ref{model:classical-adversary-provides-a-distribution-agnostic}) and the quantum online model (in Section \ref{sec:quantum_online}). Specifically, the former employs the indicator loss (mistake model), whereas the latter two involve probabilistic losses.
\end{itemize}

\begin{table}[ht]
    \centering
    \begin{tabular}{c|c|c|c}
    & & Classical & Quantum \\
    \hline
    \multirow{2}{*}{Boolean} & Realizable & $\Theta\Big( \frac{d + \log(\frac{1}{\delta})}{\epsilon}\Big)$ \footnotesize{\cite{blumer1989learnability, hanneke2016optimal}} & $\Theta\Big( \frac{d + \log(\frac{1}{\delta})}{\epsilon}\Big)$ \footnotesize{\cite{arunachalam2018optimal}}\\
    \cline{2-4}
    & Agnostic & $\Theta\Big( \frac{d + \log(\frac{1}{\delta})}{\epsilon^2}\Big)$ \footnotesize{\cite{kearns1992toward, talagrand1994sharper}} & $\Theta\Big( \frac{d + \log(\frac{1}{\delta})}{\epsilon^2}\Big)$ \footnotesize{\cite{arunachalam2018optimal}}\\
    \cline{1-4}
    \multirow{4}{*}{\raisebox{-1.5ex}{Multiclass}} & \multirow{2}{*}{\raisebox{-0.75ex}{Realizable}} & $\Omega\Big( \frac{d + \log(\frac{1}{\delta})}{\epsilon}\Big)$ \footnotesize{\cite{natarajan1989learning}} & \cellcolor{gray!10} $\Omega\Big( \frac{d + \log(\frac{1}{\delta})}{\epsilon}\Big)$  \footnotesize{(Thm. \ref{thm:quantum-PAC-k})} \\
    & & $\bigO\Big( \frac{d\log(k)\log(\frac{1}{\epsilon}) + \log(\frac{1}{\delta})}{\epsilon} \Big)$ \footnotesize{\cite{daniely2015multiclass}} & \cellcolor{gray!10} $\bigO\Big( \frac{d\log(k)\log(\frac{1}{\epsilon}) + \log(\frac{1}{\delta})}{\epsilon} \Big)$ \footnotesize{(Thm. \ref{thm:quantum-PAC-k-ub})} \\
    \cline{2-4}
    & \multirow{2}{*}{\raisebox{-0.75ex}{Agnostic}} & $\Omega\Big( \frac{d + \log(\frac{1}{\delta})}{\epsilon^2}\Big)$ \footnotesize{\cite{bendavid1995characterizations}} & \cellcolor{gray!10} $\Omega\Big( \frac{d + \log(\frac{1}{\delta})}{\epsilon^2}\Big)$ \footnotesize{(Thm. \ref{thm:quantum-PAC-k})} \\
    & & $\bigO\Big( \frac{d\log(k) + \log(\frac{1}{\delta})}{\epsilon^2} \Big)$ \footnotesize{\cite{bendavid1995characterizations}} & \cellcolor{gray!10} $\bigO\Big( \frac{d\log(k) + \log(\frac{1}{\delta})}{\epsilon^2} \Big)$ \footnotesize{(Thm. \ref{thm:quantum-PAC-k-ub})} \\
    \hline
    \end{tabular}
    \caption{An overview of sample complexity results for \textbf{batch} learning in classical and quantum paradigms. Our novel contributions are presented in boxes shaded gray. $d$ denotes the VC dimension for Boolean cases ($d := \vcDim(\cH)$), and the Natarajan dimension for Multiclass cases ($d := \nDim(\cH)$).}
    \label{table:results-in-batch-LT}
\end{table}

\begin{table}[ht]
    \centering
    \begin{tabular}{c@{\hskip 2pt}|c@{\hskip 2pt}|c@{\hskip 2pt}|c@{\hskip 2pt}|c}
    & & Classical & \cellcolor{gray!10} Classical & \cellcolor{gray!10} \\
    & & (Input-based) & \cellcolor{gray!10} (Distribution-based) & \cellcolor{gray!10} \multirow{-2}{*}{Quantum} \\
    \hline
    \multirow{2}{*}{Bool.} & Real. & $\Theta(d)$ \footnotesize{\cite{littlestone1988learning}} & \cellcolor{gray!10} $\Theta(d)$ \footnotesize{(Thms. \ref{thm:loss-upper-bdd-dist}, \ref{thm:loss-lower-bdd-dist})} & \cellcolor{gray!10} $\Theta(d)$ \footnotesize{(Thms. \ref{thm:quantum-online-boolean-lb}, \ref{thm:quantum-online-boolean-ub})} \\
    \cline{2-5}
    & Agn. & $\Theta(\sqrt{dT})$ \footnotesize{\cite{ben2009agnostic, alon2021adversarial}} & \cellcolor{gray!10} $\Theta(\sqrt{dT})$ \footnotesize{(Thms. \ref{thm:regret-upper-bdd-dist}, \ref{thm:regret-lower-bdd-dist})} & \cellcolor{gray!10} $\Theta(\sqrt{dT})$ \footnotesize{(Thms. \ref{thm:quantum-online-boolean-lb}, \ref{thm:quantum-online-boolean-ub})} \\
    \cline{1-5}
    \multirow{3}{*}{Multi.} & Real. & $\Theta(d)$ \footnotesize{\cite{daniely2015multiclass}} & \cellcolor{gray!10} $\Theta(d)$ \footnotesize{(Thms. \ref{thm:loss-upper-bdd-dist-mc}, \ref{thm:loss-lower-bdd-dist-mc})} & \cellcolor{gray!10} $\Theta(d)$ \footnotesize{(Thms. \ref{thm:quantum-online-multiclass-lb}, \ref{thm:quantum-online-multiclass-ub})} \\
    \cline{2-5}
    & \multirow{2}{*}{Agn.} & $\Omega({\sqrt{dT}})$ \footnotesize{\cite{daniely2015multiclass}} & \cellcolor{gray!10} $\Omega({\sqrt{dT}})$ \footnotesize{(Thm. \ref{thm:regret-lower-bdd-dist-mc})} & \cellcolor{gray!10} $\Omega({\sqrt{dT}})$ \footnotesize{(Thm. \ref{thm:quantum-online-multiclass-lb})} \\
    & & $\tilde{\bigO}(\sqrt{dT})$ \footnotesize{\cite{hanneke2023multiclass}} & \cellcolor{gray!10} $\bigO(\sqrt{dT \log(Tk)})$ \footnotesize{(Thm. \ref{thm:regret-upper-bdd-dist-mc})} & \cellcolor{gray!10} $\bigO(\sqrt{dT \log(Tk)})$ \footnotesize{(Thm. \ref{thm:quantum-online-multiclass-ub})} \\
    \hline
    \end{tabular}
    \caption{An overview of expected regret bounds for \textbf{online} learning in the canonical classical (adversary-provides-an-input), classical \textit{adversary-provides-a-distribution}, and quantum paradigms. Our novel contributions are presented in boxes shaded gray. $d$ denotes the Littestone dimension for Boolean cases ($d := \lDimH$), and the multiclass Littlestone dimension for Multiclass cases ($d := \mclDimH$).}
    \label{table:results-in-online-LT}
\end{table}

\subsection{Related Works}

\paragraph{Situating Our Work in the Larger Context.} In the survey on the complexity of learning quantum states \cite{anshu2024survey}, the authors identify four key avenues of current research organized into a $2 \times 2$ framework: (1) learning \textbf{all states} versus \textbf{some subclasses of states}, and (2) learning with \textbf{strong requirements} (e.g., up to small trace distance) versus \textbf{weaker requirements} (e.g., PAC learning, statistical query learning). One of these avenues, concerning ``learning classical functions encoded as [quantum] states,'' aligns with the \textit{learning some subclasses of states with weaker learning requirements} category, and it is here that the models introduced in this paper contribute to the ongoing exploration of this research landscape.

\paragraph{Batch Learning.}

Quantum learning of classical functions in the batch setting has has predominantly focused on learning Boolean functions (see \cite{arunachalam2017guest} for a survey). In our work, the multiclass extension in the batch setting (Section \ref{sec:multiclass-classification}) is inspired by an open question of \cite{arunachalam2018optimal}. The lower bound (Section \ref{sec:quantum-PAC-k-lb}) in this extension is made possible due to the presence of an explicit quantum circuit (see Figure \ref{fig:qcircuit-mc-lb}) that allows for a \textit{black box}-style reduction of the quantum multiclass case to the quantum binary classification case, enabling the existing quantum batch binary lower bound (and, techniques therewith, see Section \ref{sec:quantum-PAC-2} for a detailed treatment) in \cite{arunachalam2018optimal} to yield the desired multiclass bound. This reduction is analogous to the establishment of the lower bound in the classical batch multiclass setting (as presented in Theorem 5 of \cite{daniely2015multiclass}). While the batch classification model in this paper assumes the learner has access to copies of quantum states $\ket{\psi_c}$ (i.e., an \textit{example oracle}), with no (i.e., at most a constant) advantage over classical batch models, Grover's algorithm -- along with its quadratic quantum sample complexity advantage -- has been successfully adapted in \cite{salmon2023provable} to batch \textit{binary} classification in a stronger learning model where the learner has access to a quantum circuit $Q_c$ that generates a quantum sample $\ket{\psi_c}$ (i.e., a \textit{unitary oracle}\footnote{The sample complexity here refers to the number of calls to the unitary oracle, providing access to both $Q_c$ and its Hermitian conjugate $Q_c^\dagger$}).

\paragraph{Proper vs. Improper Learning.}

In this work, we do not explicitly address whether learning is proper or improper in the various settings. However, since most of our proofs involve reductions to the corresponding classical settings, the nature of the algorithms that provide tight bounds in those settings remains unchanged in our cases. The separation between proper and improper quantum learning has garnered significant interest in the community \cite{arunachalam2020quantum, nayak2024proper}. In particular, this separation has been studied in the context of quantum extensions of the coupon collector problem, which provides a classical example of such a separation. Notably, \cite{nayak2024proper} established a provable separation between proper and improper learning in the quantum setting using a ``padded'' variant of the quantum coupon collector problem. On the classical side, \cite{bousquet2020proper} provides a full characterization of when proper learning achieves optimal sample complexity, based on a combinatorial parameter known as the dual Helly number. To the best of our knowledge, an analogous characterization for quantum proper learning remains an open question.

\paragraph{Online Learning and Adversary-provides-a-distribution.}

The online \textit{adversary-provides-a-distribution} model (Section \ref{model:classical-adversary-provides-a-distribution}) has analogues in other works. For example, \cite{haghtalab2020smoothed} and \cite{block2022smoothed} consider a setting where an adaptive adversary provides a distribution. Their setup is designed for smoothed learning, where at each time step, the adversary provides a sample drawn from their chosen distribution. However, their framework assumes a full-information setting, where the learner observes all losses, whereas in our model, the learner only receives partial feedback. Another related connection can be found in the online game introduced in \cite{lugosi2023online} to analyze batch generalization error. While their setting, like ours, involves partial information on losses (and, therefore, regret), the probabilistic aspect of their model arises from the learner selecting a distribution over a hypothesis class, rather than the adversary providing one.

\subsection{Organization}

The paper is organized as follows. In Section \ref{sec:preliminary}, we present preliminaries, including notation, quantum and probability theory basics, and batch learning frameworks. Section \ref{sec:quantum-batch} contains one of our main results, addressing the quantum sample complexity in the batch multiclass setting. Moving on to Section \ref{sec:classical-online-learning}, we revisit the canonical online model (Section \ref{model:classical-adversary-provides-an-input}), followed by the introduction and presentation of results for the classical adversary-provides-a-distribution model in both realizable (Section \ref{model:classical-adversary-provides-a-distribution}) and agnostic (Section \ref{model:classical-adversary-provides-a-distribution-agnostic}) settings for both binary and multiclass (Section \ref{sec:classical-adversary-provides-a-distribution-mc}) classification. This model serves as an intermediary for transitioning from the canonical classical online model to the quantum online model. In Section \ref{sec:quantum_online}, we introduce the quantum online learning model and summarize results within the established framework. The paper concludes with a discussion and reflection on the obtained results in Sections \ref{sec:takeaways} and \ref{sec:conclusion}, along with some open questions for future exploration.

\section{Preliminaries}\label{sec:preliminary}

In this section, we introduce the key concepts and notation used throughout the paper. First, we briefly review the fundamentals of quantum computing. Next, we present essential probability theory concepts needed for the proofs in Section \ref{sec:classical-online-learning}. Finally, we outline the batch learning frameworks in both classical and quantum settings, setting the stage for Section \ref{sec:quantum-batch}.

\subsection{Notation}\label{sec:notation}

In the bra-ket (Dirac) notation, a ket, $\ket{x}$, denotes a column vector in a complex vector space with an inner product (i.e., a Hilbert space). It is used primarily in the context of describing the state of a quantum system (e.g., see Definition \ref{def:qubit}). A bra $\bra{\cdot}$ is the dual of the ket, in that $\bra{x} = \ket{x}^\dagger$, where the $\dagger$ operator denotes the conjugate transpose. Typically, the bra notation is used for operators $\bra{M}$ (e.g. measurement operators) acting on a ket. This notation lends itself naturally to the notion of inner product $\braket{x|x} = \|x\|^2$, and matrix-vector multiplication $\braket{M|x}$. Furthermore, note that $\ket{x,y}$ denotes the tensor product $\ket{x} \otimes \ket{y}$, where $\otimes$ denotes the standard tensor product of two vector spaces. The comma may be omitted, and we have numerous equivalent notations for the tensor product: e.g., $\ket{0^2} = \ket{00} = \ket{0,0} = \ket{0}\ket{0} = \ket{0} \otimes \ket{0} = \ket{0}^{\otimes 2}$.

\subsection{Quantum Basics}

Analogous to how a classical bit (\textit{bit}) is a unit of classical information, a quantum bit (\textit{qubit}) is a unit of quantum information. The difference between the two is best illustrated by considering how each is realized. A bit is realized via well separated values of a physical property of a system (e.g., voltage across an element in an electric circuit). If the value is higher than a certain threshold, the bit assumes the value 1. Otherwise, it assumes the value 0. Thus, a bit carries the information equivalent of its namesake, a binary digit. A qubit, on the other hand, is realized as a two-level quantum system; e.g., as the spin (up, down) of an electron, the polarization (horizontal, vertical) of a photon, or the discrete energy levels (ground, excited) of an ion. Consequently, it is governed by the postulates of quantum mechanics \cite{nielsen2010quantum}, as detailed in the following paragraphs.

The \textit{state space} of a qubit is a 2-dimensional complex vector space, denoted as $\mathbb{C}^2$. The definition of a qubit as a \textit{state vector} within this space is presented in the following definition.

\begin{definition}[Qubit] \label{def:qubit}
    A single (isolated) qubit is described by a state vector $\ket{\psi}$, which is a unit vector in the state space $\mathbb{C}^2$. Mathematically,
    \[
    \ket{\psi} = \alpha_0 \ket{0} + \alpha_1 \ket{1}, \quad \alpha_0, \alpha_1 \in \mathbb{C}, \quad |\alpha_0|^2 + |\alpha_1|^2 = 1,
    \]
    where $\ket{0} = \begin{bmatrix} 1\\ 0 \end{bmatrix}$ and $\ket{1} = \begin{bmatrix} 0\\ 1 \end{bmatrix}$ are basis vectors for the state space.
\end{definition}

So, although we have two basis states (much as we did for the classical bit), the qubit is allowed to be in a (complex) superposition of the two, whereas a classical bit must deterministically be in one of the basis states. Additionally, as our learning examples (refer to (\ref{example:quantum-realizable}), (\ref{example:quantum-agnostic})) will involve multiple qubits, it is important to note that the state space of the composite system, comprising many qubits, is the tensor product of the state spaces of its components (i.e., the individual qubits). The joint state of the composite system formed by $n$ qubits, each in state $\ket{\psi_i}, \; i \in \{1,\ldots,n\}$, is given by, $\ket{\psi_1} \otimes \ket{\psi_2} \otimes \cdots \otimes \ket{\psi_n} \in \mathbb{C}^{2^n}$. Now, let us express the state vector for multi-qubit system in terms of the standard basis elements.

\begin{definition}[Multi-qubit system]
The state vector, $\ket{\Psi} \in \mathbb{C}^{2^n}$, describing the system of $n$ qubits can be expressed in terms of the standard basis elements $\{\ket{b} = \ket{b_1} \otimes \cdots \otimes \ket{b_n} | \, b=(b_1,\ldots,b_n) \in \{0,1\}^n \}$ as follows:
\[
    \ket{\Psi} = \sum_{b \in \{0,1\}^n} \alpha_b \ket{b},
\]
where $\alpha_b = \braket{\Psi|b} \in \mathbb{C}$, $\sum_b |\alpha_b|^2 = 1$.
\end{definition}

In contrast, the joint state of $n$ classical bits is described by their Cartesian product. This essential distinction between Cartesian and tensor products is precisely the phenomenon of quantum entanglement, namely the existence of (pure) states of a composite system that are not product states of its parts. Quantum entanglement, alongside superposition, lies at the heart of intrinsic advantages of quantum computing.

Any manipulation of a quantum system is confined to unitary evolution. In the context of computation, this implies that \textit{all quantum gates are unitary operators}, restricting their application to reversible computations. An avenue for irreversible computation, and the only way to obtain classical outputs in the quantum realm, is the notion of a \textit{measurement}.

\begin{definition}[Measurement]\label{def:measurement}
    Quantum measurements are described by a collection $\{M_m\}$ of measurement operators acting on the state space of the system. The index $m$ denotes the possible classical outcomes of the measurement. If the quantum system is in the state $\ket{\psi}$ before measurement, then the probability that result $m$ occurs is given by
    $p(m) = \bra{\psi}M_m^\dagger M_m \ket{\psi}$,
    and the state of the system after the measurement, if $m^\star$ is observed, ``collapses to''
    $M_{m^\star} \ket{\psi} / \sqrt{p(m^\star)}$.
    To ensure the conservation of total probability, $\sum_m M_m^\dagger M_m = I$ is satisfied.
\end{definition}

\noindent Here are a couple of examples to illustrate the above definition:
\begin{itemize}[noitemsep,nolistsep]
    \item A measurement in the standard basis is implemented by measurement operators $M_0 = \ket{0}\bra{0}$ and $M_1 = \ket{1}\bra{1}$.
    \item A measurement in the standard basis of the state $\ket{\Psi} = \sum_{b \in \{0,1\}^n} \alpha_b \ket{b}$ yields the classical outcome $b$ with probability $|\alpha_b|^2$.
\end{itemize}

\subsection{Probability Theory Basics}

Let $\mathbf{X} := (X_t)_{t \geq 1}$ be a sequence of random variables defined on a probability space $(\Omega, \mathcal{F}, \mathbb{P})$, where each $X_t$ represents the outcome of an experiment at time $t$. Associated with this sequence is a filtration $\mathbb{F} := (\mathcal{F}_t)_{t \geq 1}$, an increasing family of $\sigma$-algebras such that $\mathcal{F}_t$ captures the information available up to time $t$. Formally, $\mathcal{F}_1 \subseteq \mathcal{F}_2 \subseteq \cdots \subseteq \mathbb{F}$, where $\mathcal{F}_t$ contains events determined by $\mathbf{X}|_t = (X_1, X_2, \dots, X_t)$.

\begin{definition}[Martingale]\label{def:martingale}
    A sequence of random variables $\mathbf{M} := (M_t)_{t \geq 1}$ is said to be a \textit{martingale} with respect to the filtration $\mathbb{F} := (\mathcal{F}_t)_{t \geq 1}$ if it satisfies the following conditions:
    \begin{enumerate}[noitemsep]
        \item Adaptation: $M_t$ is $\mathcal{F}_t$-measurable for all $t \geq 1$ (i.e., we say $\mathbf{M}$ is adapted to $\mathbb{F}$),
        \item Integrability: $\mathbb{E}[|M_t|] < \infty$ for all $t \geq 1$,
        \item Martingale Property: $\mathbb{E}[M_{t+1} | \mathcal{F}_t] = M_t$ for all $t \geq 1$.
    \end{enumerate}
\end{definition}

\begin{definition}[Martingale Difference Sequence]\label{def:MDS}
    Given a martingale $\mathbf{M} := \{M_t\}_{t \geq 1}$, the corresponding \textit{martingale difference sequence} is defined as $\mathbf{D} := \{D_{t+1}\}_{t \geq 1}$, where $D_{t+1} = M_{t+1} - M_t$. By the martingale property, the sequence $\mathbf{D}$ satisfies $\mathbb{E}[D_{t+1} | \mathcal{F}_t] = 0$ for all $t \geq 1$, indicating that each increment $D_{t+1}$ has conditional mean zero given the information up to time $t$.

    In general, any sequence of random variables $\mathbf{X} := \{X_t\}_{t \geq 1}$ is a martingale difference sequence with respect to the filtration $\mathbb{F} := (\cF_{t})_{t \geq 1}$ if it satisfies the following conditions:
    \begin{enumerate}[noitemsep]
        \item Adaptation: $X_t$ is $\mathcal{F}_t$-measurable for all $t \geq 1$ (i.e. $\mathbf{X}$ is adapted to $\mathbb{F}$),
        \item Integrability: $\mathbb{E}[|X_t|] < \infty$ for all $t \geq 1$,
        \item Conditional Mean-Zero Property: $\mathbb{E}[X_t | \mathcal{F}_{t-1}] = 0$ for all $t \geq 1$\footnote{By convention, we treat the ``conditioning'' on $\cF_0$ as an unconditional expectation, i.e., $\mathbb{E}[X_1]=0$.}.
    \end{enumerate}
\end{definition}

\subsection{PAC Learning Framework}

In the classical PAC (Probably Approximately Correct) learning model \cite{valiant1984theory}, a learner is provided oracle access to samples $(x, y)$, where $x$ is sampled from some unknown distribution $D$ on $\cX$ and $y = h^\star(x)$, for some \textit{target} hypothesis $h^\star : \cX \to \cY$. We assume that $h^\star \in \cH$, where $\cH$ is a predefined hypothesis class, i.e., the learner has prior knowledge of $\cH$. The goal of the learning problem is to find\footnote{Note that $h$ need not necessarily belong to $\cH$. If it does, the learner is called \textit{proper}. If not, the learner is \textit{improper}. \label{footnote:proper-improper}} $h: \cX \to \cY$ such that the generalization error, given by the loss function $\cL(h, D, h^\star) = \Prob_{x \sim D} (h(x) \neq h^\star(x))$, is minimized.

\begin{definition}[PAC learner]\label{def:PAC-learner}
    An algorithm $\cA$ is an $(\epsilon,\delta)$-PAC learner for a hypothesis class $\cH$ if, for any unknown distribution $D$ and for all $h^\star \in \cH$, $\cA$ takes in $m$ pairs of labeled instances, i.e., $\{(x_i,h^\star(x_i))\}_{i=1}^m$, each drawn i.i.d. from $D$, and outputs a hypothesis $h$ such that $\Prob[{\cL(h, D, h^\star) \leq \epsilon}] \geq 1 - \delta$, where the outer probability is over the sequence of examples and the learner’s internal randomness.
\end{definition}

Indeed, an $(\epsilon,\delta)$-PAC learner outputs a hypothesis that is, with high probability ($\geq 1 - \delta$), approximately correct ($\cL \leq \epsilon$). A hypothesis class $\cH$ is \textit{PAC-learnable} if there exists an algorithm $\cA$ that is an $(\epsilon,\delta)$-PAC learner for $\cH$.
When $\cY = \{0,1\}$, we are in the setting of binary classification. To express the sample complexity of learning boolean function classes later on, we define below a key combinatorial parameter known as the VC dimension.

\begin{definition}[VC dimension]\label{def:VC-dimension}
    Given a hypothesis class $\cH = \{h:\cX \to \{0,1\}\}$, a set $S = \{s_1,\ldots,s_t\}\subseteq \cX$ is said to be shattered by $\cH$ if, for every labeling $\ell \in \{0,1\}^t$, there exists an $h \in \cH$ such that $(h(s_1),h(s_2),\ldots,h(s_t)) = \ell$. The VC dimension of $\cH$, $\vcDim(\cH)$, is the size of the largest set $S$ that is shattered by $\cH$.
\end{definition}

\subsection{Agnostic Learning Framework}

In the PAC learning framework, we worked with the \textit{realizability assumption}, namely that $h^\star \in \cH$. If we omit this rather strong assumption, we are able to generalize the PAC learning framework to the agnostic learning framework \cite{kearns1992toward}.
Here, a learner is provided with oracle access to samples $(x,y)$, sampled from some unknown distribution $D$ on $\cX \times \cY$. The learner has knowledge of a predefined hypothesis class $\cH$. The objective of the learning problem is to find$^{\ref{footnote:proper-improper}}$ $h: \cX \to \cY$ such that the \textit{regret}
\[
\cR(h,D) = \Prob_{(x,y) \sim D} (h(x) \neq y) - \inf_{h_c \in \cH} \Prob_{(x,y) \sim D} (h_c(x) \neq y),
\]
is minimized. One can notice that if the labels happen to satisfy some $h^\star \in \cH$, $\cR \equiv \cL$.

\begin{definition}[Agnostic learner]\label{def:Agnostic-learner}
An algorithm $\cA$ is an $(\epsilon, \delta)$-agnostic learner for a hypothesis class $\cH$ if, for any unknown distribution $D$, $\cA$ takes in $m$ pairs of labeled instances, i.e., ${(x_i, y_i)}_{i=1}^m$, each drawn i.i.d. from $D$, and outputs a hypothesis $h$ such that $\Prob[{\cR(h, D) \leq \epsilon}] \geq 1 - \delta$, where the outer probability is over the sequence of examples and the learner’s internal randomness.
\end{definition}
    
\subsection{Quantum PAC and Agnostic Learning Frameworks} \label{framework:quantum-pac-and-agnostic-quantum-example}

In the quantum setting, the primary difference from the classical setting lies in how the examples are provided. In particular, in the PAC learning setup, a quantum example \cite{bshouty1995learning} takes the form
\begin{equation} \label{example:quantum-realizable}
\sum_{x \in \{0,1\}^n} \sqrt{D(x)} \ket{x,h^\star(x)},
\end{equation}
for some $h^\star \in \cH$, where $D: \{0,1\}^n \to [0,1]$ is a distribution over the instance space\footnote{Here, we have taken $\cX = \{0,1\}^n$ for convenience and ease of analysis. However, any finite $\cX$ could be mapped to this one, if needed.}, as before. This might appear slightly strange, as a single example seemingly contains information about \textit{all} possible classical examples. However, if we view it via the lens of measurement (see Definition \ref{def:measurement}), then it is clear that measuring a quantum example will provide the learner with a \textit{single} classical example $(x,h^\star(x))$ with probability $D(x)$, exactly how it was in the classical PAC learning setup. While we have argued that the quantum example is a natural generalization of the classical example, the question still remains as to whether any sample complexity advantages in the quantum realm arise from the intrinsic description of a quantum example or from the quantum algorithm used or from both.

In the agnostic learning setting, a quantum example takes the form
\begin{equation} \label{example:quantum-agnostic}
\sum_{(x,y) \in \{0,1\}^{n+1}} \sqrt{D(x,y)} \ket{x,y},
\end{equation}
where, now, $D: \{0,1\}^{n+1} \to [0,1]$. These examples, like in the quantum PAC setting above, are typically \textit{prepared} by acting on the all-zero state $\ket{0^n,0}$ via an appropriate quantum circuit.

Given quantum examples (instead of classical examples), Definitions \ref{def:PAC-learner} and \ref{def:Agnostic-learner} otherwise stay exactly the same in the quantum setting.

\section{Quantum Batch Learning}\label{sec:quantum-batch}

Under the quantum (batch) learning frameworks outlined in Section \ref{framework:quantum-pac-and-agnostic-quantum-example}, we investigate the sample complexity of batch learning a hypothesis class $\cH$. Specifically, we address the question of how many copies of quantum examples, as given in (\ref{example:quantum-realizable}) (resp. (\ref{example:quantum-agnostic})), are required to $(\epsilon, \delta)$-quantum PAC (resp. quantum agnostic) learn $\cH$.

\subsection{Binary Classification}\label{sec:quantum-PAC-2}

In the binary classification setting, this question has been conclusively answered, and we reproduce the corresponding theorem below.

\begin{theorem}[Sample complexity bounds\footnote{The upper bounds are obtained trivially via a measure-and-learn-classically quantum learner, whereas matching lower bounds are provided by \cite{arunachalam2018optimal}.} for quantum batch binary classification; Theorems 23 and 25 in \cite{arunachalam2018optimal}] \label{thm:quantum-PAC-2}
    Let $\cH \subseteq \{0,1\}^\cX$. For every $\delta \in (0,1/2)$ and $\epsilon \in (0,1/20)$, the sample complexity of an $(\epsilon, \delta)$-quantum PAC learner (and, respectively, an $(\epsilon, \delta)$-quantum agnostic learner) for the hypothesis class $\cH$ is given by:
    \[
        m^{\text{PAC}} = \Theta\Bigg( \frac{\vcDim(\cH) + \log(\frac{1}{\delta})}{\epsilon} \Bigg), \;\; \text{and} \;\;\; m^{\text{agnostic}} = \Theta\Bigg( \frac{\vcDim(\cH) + \log(\frac{1}{\delta})}{\epsilon^2} \Bigg).
    \]
\end{theorem}

Arunachalam and de Wolf \cite{arunachalam2018optimal} established the \textit{optimal} lower bounds via quantum state identification, employing ideas from Fourier analysis to assess the performance of the Pretty Good Measurement. They also derived \textit{near-optimal} information-theoretic lower bounds, which are off by a factor of $\frac{1}{\log(\vcDim(\cH)/\epsilon)}$ in the term involving $\vcDim(\cH)$ in both the PAC and agnostic settings. Their information-theoretic approach has since been refined to provide the \textit{optimal} lower bounds \cite{hadiashar2023optimal}.

Having resolved the quantum sample complexity in the (batch) binary classification setting, \cite{arunachalam2018optimal} also presented an open question regarding the quantum sample complexity in the (batch) \textit{multiclass} classification setting, i.e., when $\cH \subseteq \cY^\cX$ with $|\cY| = k > 2$.

\subsection{Multiclass Classification}\label{sec:multiclass-classification}

In this subsection, we provide an answer to the aforementioned question. To express sample complexity results in this setting, we first define the combinatorial parameter, Natarajan dimension ($\nDim(\cdot)$), which serves as a generalization of the VC dimension to the multiclass setting.

\begin{definition}[Natarajan dimension]
    Given a hypothesis class $\cH = \{h:\cX \to [k]\}$, a set $S = \{s_1,\ldots,s_t\}\subseteq \cX$ is said to be N-shattered by $\cH$ if there exist two ``witness'' functions $f_0, f_1: S \to [k]$ such that:
    \begin{itemize}[noitemsep]
        \item For every $x \in S$, $f_0(x) \neq f_1(x)$.
        \item For every $R \subseteq S$, there exists a function $h \in \cH$ such that
        \[
        \forall x \in R, h(x) = f_0(x) \; \text{and} \; \; \forall x \in S \setminus R, h(x) = f_1(x).
        \]
    \end{itemize}
    The Natarajan dimension of $\cH$, $\nDim(\cH)$, is the size of the largest set $S$ that is N-shattered by $\cH$.
\end{definition}

\subsubsection{Lower Bounds} \label{sec:quantum-PAC-k-lb}

\begin{theorem}[Sample complexity lower bounds for quantum batch multiclass classification] \label{thm:quantum-PAC-k}
Let $\cH \subseteq \cY^\cX$, with $|\cY| = k > 2$. For every $\delta \in (0,1/2)$ and $\epsilon \in (0,1/20)$, the sample complexity of an $(\epsilon, \delta)$-quantum PAC learner (and, respectively, an $(\epsilon, \delta)$-quantum agnostic learner) for the hypothesis class $\cH$ is bounded below as follows:
\[
m^{\text{PAC}} = \Omega\Bigg( \frac{\nDim(\cH) + \log(\frac{1}{\delta})}{\epsilon} \Bigg), \;\; \text{and} \;\;\; m^{\text{agnostic}} = \Omega\Bigg( \frac{\nDim(\cH) + \log(\frac{1}{\delta})}{\epsilon^2} \Bigg).
\]
\end{theorem}

At its core, the proof involves reducing the problem to the quantum binary case -- establishing that a learning algorithm for $\mathcal{H}$ implies a learning algorithm for $\mathcal{H}_d$, where $\vcDim(\mathcal{H}_d) = \nDim(\mathcal{H}) = d$. This, in turn, enables us to deduce a sample complexity lower bound for learning $\mathcal{H}$ based on the corresponding lower bound for learning $\mathcal{H}_d$. A key step in the reduction involves the following transformation of a quantum binary example\footnote{To maintain consistency with Section \ref{framework:quantum-pac-and-agnostic-quantum-example}, the input space $[d]$ can be identified with $\{0,1\}^{\lceil \log_2 d \rceil}$.} into a quantum multiclass example,
\begin{equation}\label{eq:quantum-binary-to-multiclass-example}
\sum_{x \in [d]} \sqrt{D(x)} \ket{x,y} \to \sum_{x \in [d]} \sqrt{D(x)} \ket{x,f_y(x)}, \; \text{where} \; y \in \{0,1\}, \; \text{and} \; f_0, f_1 : [d] \to [k].
\end{equation}

While in the corresponding classical reduction proof, converting $(x,y) \to (x,f_y(x))$ is entirely trivial with the knowledge of $x, y, f_0, f_1$, performing the transformation in (\ref{eq:quantum-binary-to-multiclass-example}) using only unitary operations (in a reversible manner) in the quantum realm involves delicate reasoning using an explicit quantum circuit. In particular, it is noteworthy as its existence hinges on the reversibility of the transformation $y \leftrightarrow f_y$, which is guaranteed precisely due to the definition of N-shattering.

As preliminaries for the proof, we first introduce the quantum \texttt{X}, \texttt{CNOT}, \texttt{TOFFOLI} gates, and quantum oracles for computing classical functions. The notation $\oplus$ refers to the classical \texttt{XOR} operation (i.e., addition modulo 2).

\begin{definition}[\texttt{X} gate]
    \texttt{X} (or the Pauli-\texttt{X}) gate is the quantum equivalent of the classical \texttt{NOT} gate. It operates on one qubit, mapping $\ket{0} \to \ket{1}$ and $\ket{1} \to \ket{0}$ (i.e. it ``flips'' the qubit).
\end{definition}

\begin{definition}[\texttt{CNOT} gate]
    \texttt{CNOT} is a quantum gate that operates on two qubits, one control and one target. If the control qubit is in the state $\ket{1}$, it flips (i.e., applies an \texttt{X} gate to) the target qubit.
\end{definition}

\begin{definition}[\texttt{TOFFOLI} gate]
    \texttt{TOFFOLI} is a quantum gate that operates on three qubits, two control and one target. If the control qubits are both in the state $\ket{1}$, it flips (i.e., applies an \texttt{X} gate to) the target qubit.
\end{definition}

\begin{definition}[Quantum oracle $U_f$]
    For classical functions $f: \{0,1\}^m \to \{0,1\}^n$, there exists\footnote{In fact, the quantum oracle $U_f$ can be implemented in a rather straightforward way, by using the truth table of $f$ and generalizations of the \texttt{CNOT} gate that use several qubits as controls.} a quantum oracle $U_f$ that performs the unitary evolution
    \[
    U_f \ket{x,y} = \ket{x,y \oplus f(x)},
    \]
    for $x \in \{0,1\}^m$ and $y \in \{0,1\}^n$.
\end{definition}

\begin{proof}{(of Theorem \ref{thm:quantum-PAC-k})}
Let $\cH \subseteq \cY^\cX$ be a hypothesis class of Natarajan dimension $d$ and let $\cH_d = \{0,1\}^{[d]}$. Let $\cA$ be a \textit{quantum} PAC (corresp. \textit{quantum} agnostic) learning algorithm for $\cH$. We proceed to show that it is possible to construct a quantum PAC (corresp. quantum agnostic) learning algorithm, $\bar{\cA}$, for $\cH_d$. Therefore, by reduction, we would obtain $m_{\bar{\cA},\cH_d} \leq m_{\cA, \cH}$, and thus $m^{\text{PAC}}_{\cH_d} \leq m^{\text{PAC}}_{\cH}$ (corresp. $m^{\text{agnostic}}_{\cH_d} \leq m^{\text{agnostic}}_{\cH}$). Since, by construction, $\vcDim(\cH_d) = d = \nDim(\cH)$, the reduction allows us to obtain the sample complexity lower bounds being proven here (for the multiclass case), from the corresponding lower bounds for quantum batch binary classification (Theorem \ref{thm:quantum-PAC-2}). Now, for the key step of the proof, given a \textit{quantum} learner $\cA$ for $\cH$, it is possible to construct a \textit{quantum} learner $\bar{\cA}$, for $\cH_d$, as follows. We show this for the \textit{quantum PAC} case, and comment here that this reduction in the \textit{quantum agnostic} case will proceed identically.

The learner $\bar{\cA}$ receives $m$-copies of the quantum example $\sum_{x \in [d]} \sqrt{D(x)} \ket{x,y}$, where $(x,y) \in [d] \times \{0,1\}$ and $D : [d] \to [0,1]$ is an \textit{arbitrary} distribution on $[d]$. Now, let $S = \{s_1,\ldots,s_d\} \subseteq \cX$ be a set and $f_0, f_1$ be the functions that witness the N-shattering of $S$ by $\cH$. The learner $\bar{\cA}$ will now attempt to convert\footnote{The learner $\bar{\cA}$ will then present these transformed examples to $\cA$, the quantum PAC learner for $\cH$.} each of its $m$-copies of $\sum_{x \in [d]} \sqrt{D(x)} \ket{x,y}$ to $\sum_{x \in [d]} \sqrt{D(x)} \ket{s_x,f_y(s_x)}$. However, as $s_x$ is simply an indexing of the elements of the set $S$, without loss of generality, we let $\bar{\cA}$ convert each of its $m$-copies of $\sum_{x \in [d]} \sqrt{D(x)} \ket{x,y}$ to $\sum_{x \in [d]} \sqrt{D(x)} \ket{x,f_y(x)}$ instead. We claim that the transformation
\begin{equation}\label{eq:transformation-reduction-proof}
\ket{\psi}_{\bar{\cA}} = \sum_{x \in [d]} \sqrt{D(x)} \ket{x,y} \mapsto \sum_{x \in [d]} \sqrt{D(x)} \ket{x,f_y(x)} = \ket{\psi}_{\cA},
\end{equation}
is attainable. Indeed, the quantum circuit shown in Figure \ref{fig:qcircuit-mc-lb} (and described subsequently) performs the following augmented transformation,
\begin{equation}\label{eq:transformation-reduction-proof-augmented}
\overline{\ket{\psi}_{\bar{\cA}}} = \sum_{x \in [d]} \sqrt{D(x)} \ket{x,y,0^{3\lceil \log_2 k \rceil}} \mapsto \sum_{x \in [d]} \sqrt{D(x)} \ket{x, 0, 0^{2\lceil \log_2 k \rceil}, f_y(x)} = \overline{\ket{\psi}_{\cA}}.
\end{equation}

\begin{figure}[ht]
    \centering
    \includegraphics[width=\linewidth]{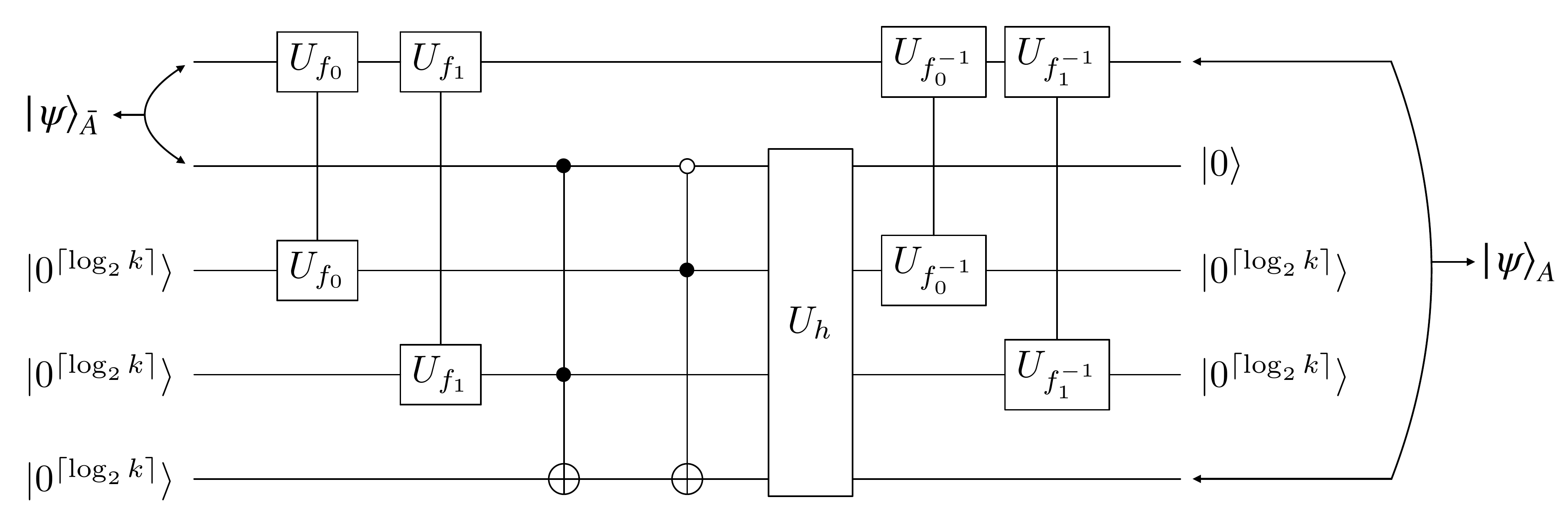}
    \caption{A quantum circuit that performs the transformation in (\ref{eq:transformation-reduction-proof-augmented}). The first row denotes the $\lceil \log_2 d \rceil$ qubits involved in encoding $x$. The second row denotes the single qubit involved in encoding $y$. The third and fourth row each denote the $\lceil \log_2 k \rceil$ ancillary qubits designed to hold the intermediate computation of $f_0(x)$ and $f_1(x)$ respectively. The fifth row denotes the $\lceil \log_2 k \rceil$ qubits designed to hold the output $f_y(x)$.}
    \label{fig:qcircuit-mc-lb}
\end{figure}

We form the augmented state $\overline{\ket{\psi}_{\bar{\cA}}}$ by appending $3 \lceil \log_2 k \rceil$ qubits in the state $\ket{0}$ to the input $\ket{\psi}_{\bar{\cA}}$. We intend to use one set of $\lceil \log_2 k \rceil$ qubits to encode each of $f_0(x)$, $f_1(x)$ (both ancillary) and $f_y(x)$ (solution). First, we pass the qubits encoding $x$ and one set of $\lceil \log_2 k \rceil$ qubits in the state $\ket{0}$ to the quantum oracle $U_{f_0}$. From this, we transform
\[
\overline{\ket{\psi}_{\bar{\cA}}} = \sum_{x \in [d]} \sqrt{D(x)} \ket{x,y,0^{3\lceil \log_2 k \rceil}} \mapsto \sum_{x \in [d]} \sqrt{D(x)} \ket{x, y, f_0(x), 0^{2\lceil \log_2 k \rceil}}.
\]
Next, we pass the qubits encoding $x$ and another set of $\lceil \log_2 k \rceil$ qubits in the state $\ket{0}$ to the quantum oracle $U_{f_1}$. From this, we transform
\[
\sum_{x \in [d]} \sqrt{D(x)} \ket{x, y, f_0(x), 0^{2\lceil \log_2 k \rceil}} \mapsto \sum_{x \in [d]} \sqrt{D(x)} \ket{x, y, f_0(x), f_1(x), 0^{\lceil \log_2 k \rceil}}.
\]
Now, we apply $\lceil \log_2 k \rceil$ \texttt{TOFFOLI} gates to each set of the qubit encoding $y$, a qubit involved in encoding $f_1(x)$ and a qubit $\ket{0}$ in the set of $\lceil \log_2 k \rceil$ remaining qubits that have not yet been operated on (that are designed hold the final result $f_y(x)$). From this, we transform
\[
\sum_{x \in [d]} \sqrt{D(x)} \ket{x, y, f_0(x), f_1(x), 0^{\lceil \log_2 k \rceil}} \mapsto \sum_{x \in [d]} \sqrt{D(x)} \ket{x, y, f_0(x), f_1(x), y.f_1(x)}.
\]
Next, we apply the \texttt{X} gate\footnote{The concise description for controlling on $\ket{y} = \ket{0}$ (instead of explicitly on $\ket{1-y} = \ket{1}$) is depicted in Figure \ref{fig:qcircuit-mc-lb}.} to the qubit encoding $y$, and then apply $\lceil \log_2 k \rceil$ \texttt{TOFFOLI} gates to each set of the qubit (now encoding) $1-y$, a qubit involved in encoding $f_0(x)$ and a qubit that is holding the final result. Finally, we apply the \texttt{X} gate again to the qubit encoding $1-y$ to revert the effect of the first \texttt{X} gate. All together, we have transformed
\begin{align*}
\sum_{x \in [d]} \sqrt{D(x)} \ket{x, y, f_0(x), f_1(x), y.f_1(x)} & \mapsto \sum_{x \in [d]} \sqrt{D(x)} \ket{x, 1-y, f_0(x), f_1(x), y.f_1(x)}\\
& \mapsto \sum_{x \in [d]} \sqrt{D(x)} \ket{x, 1-y, f_0(x), f_1(x), y.f_1(x) \oplus (1-y).f_0(x)}\\
& \mapsto \sum_{x \in [d]} \sqrt{D(x)} \ket{x, y, f_0(x), f_1(x), y.f_1(x) \oplus (1-y).f_0(x)}\\
& = \sum_{x \in [d]} \sqrt{D(x)} \ket{x, y, f_0(x), f_1(x), f_y(x)},
\end{align*}
where, in the last line, we recognize that $y.f_1(x) \oplus (1-y).f_0(x) = f_y(x)$.

Now that we have computed our solution, $f_y(x)$, to ensure we are not introducing extraneous outputs to satisfy reversibility of unitary computation, we must \textit{uncompute} all ancillary qubits to their original form. Furthermore, we also do not want the input $\ket{y}$ in our output\footnote{We do not want to allow $\cA$ to be able to cheat by providing it with this additional knowledge.}, and would like to ``remove'' it, by transforming it to the $\ket{0}$ state. We proceed with the \textit{uncomputation} as follows. First, we claim, with the following subproof, that there exists a (genuine) boolean \textit{function} $h$ that takes as input $f_0(x), f_1(x)$, and $f_y(x)$ and outputs $y$.
\begin{quote}
For a short (sub)proof-by-contradiction, consider a particular input $f_0(x), f_1(x)$, and $f_y(x) = y.f_1(x) \oplus (1-y).f_0(x)$ (for some $x$) that maps to both $y=0$ and $y=1$. This would imply that $f_0(x) = f_1(x)$ for that particular $x$. However, as $f_0$ and $f_1$ are witnesses of the N-shattering of $\cH$, they must disagree on all inputs $x$, providing us with the contradiction.
\end{quote}
As $h$ is a genuine (boolean) function, we can apply the oracle $U_h$ to the qubits encoding $f_0(x), f_1(x)$, and $f_y(x)$ and the qubit encoding $y$. This gives us the transformation\footnote{In other words, we have been able to ``remove'' $y$. It is important to note that this was possible precisely because $y$ was recoverable from $f_y$ (rendering $y \leftrightarrow f_y$ reversible), ensuring its ability to be encoded in a unitary operation.}
\begin{align*}
\sum_{x \in [d]} \sqrt{D(x)} \ket{x, y, f_0(x), f_1(x), f_y(x)} & \mapsto \sum_{x \in [d]} \sqrt{D(x)} \ket{x, y \oplus y, f_0(x), f_1(x), f_y(x)}\\
& = \sum_{x \in [d]} \sqrt{D(x)} \ket{x, 0, f_0(x), f_1(x), f_y(x)}.
\end{align*}

Lastly, we apply $U_{f_0^{-1}}$ to the qubits encoding $x$ and the ancillary qubits encoding $f_0(x)$. And then, we apply $U_{f_1^{-1}}$ to the qubits encoding $x$ and the ancillary qubits encoding $f_1(x)$. All together, we have transformed,
\begin{align*}
\sum_{x \in [d]} \sqrt{D(x)} \ket{x, 0, f_0(x), f_1(x), f_y(x)} & \mapsto \sum_{x \in [d]} \sqrt{D(x)} \ket{x, 0, 0^{\lceil \log_2 k \rceil}, f_1(x), f_y(x)}\\
& \mapsto \sum_{x \in [d]} \sqrt{D(x)} \ket{x, 0, 0^{2 \lceil \log_2 k \rceil}, f_y(x)} = \overline{\ket{\psi}_{\cA}}.
\end{align*}

Now that we have performed the transformation in (\ref{eq:transformation-reduction-proof-augmented}) quantumly, ignoring\footnote{They are each deterministically $\ket{0}$.} ancillary (and the ``removed'' $y$) qubits in the output, we note that the output of the transformation in (\ref{eq:transformation-reduction-proof}) is attained.

The learner $\bar{\cA}$ now feeds the $m$-copies of $\sum_{x \in [d]} \sqrt{D(x)} \ket{x,f_y(x)}$ as input to the \textit{quantum} learner $\cA$ and obtains (from $\cA$) a \textit{classical} (black box) function\footnote{In fact, our identification $s_x \leftrightarrow x$ earlier means that $g:[d] \to \cY$ instead, and $f: [d] \to \{0,1\}$ would then be given by $f(i) = 1$ if and only if $g(i) = f_1(i)$.} $g: S \to \cY$ (recall $S = \{s_1, \ldots, s_d\} \subseteq \cX$ is the subset that is N-shattered by $\cH$). Finally, $\bar{\cA}$ outputs the hypothesis (as a black box function) $f: [d] \to \{0,1\}$, given by $f(i) = 1$ if and only if $g(s_i) = f_1(s_i)$ (which by construction learns $\cH_d$).

\end{proof}

\subsubsection{Upper Bounds}\label{sec:classical-quantum-multiclass-PAC-ub-proof}

In general, classical sample complexity upper bounds trivially translate to the corresponding quantum ones, as the quantum learner always has the option of simply performing a measurement on each quantum example, and perform the classical learning algorithm on the resulting $m$ classical examples. We include the theorem statement (Theorem \ref{thm:quantum-PAC-k-ub}) and proof below for completeness.

\begin{theorem}[Sample complexity upper bounds for quantum batch multiclass classification]\label{thm:quantum-PAC-k-ub}
Let $\cH \subseteq \cY^\cX$, with $|\cY| = k > 2$. The sample complexity of an $(\epsilon,\delta)$-quantum PAC learner (and, respectively, an $(\epsilon, \delta)$-quantum agnostic learner) for the hypothesis class $\cH$ is bounded above as follows:
\[
m^{\text{PAC}} = \bigO\Bigg( \frac{\nDim(\cH)\log(k)\log(\frac{1}{\epsilon}) + \log(\frac{1}{\delta})}{\epsilon} \Bigg), \;\; \text{and} \;\;\; m^{\text{agnostic}} = \bigO\Bigg( \frac{\nDim(\cH)\log(k) + \log(\frac{1}{\delta})}{\epsilon^2} \Bigg).
\]
\end{theorem}

\begin{proof}
    The quantum PAC (resp. quantum agnostic) learner performs a measurement on each of the $m$ examples $\sum_x \sqrt{D(x)} \ket{x,y}$ (corresp. $\sum_{x,y} \sqrt{D(x,y)} \ket{x,y}$) in the standard computational basis. This provides $m$ classical examples, i.e. gives us the training set $\{(x_i,y_i)\}_{i=1}^m$, where a given $(x_i,y_i)$ appears with probability $D(x_i)$ (corresp. $D(x_i,y_i)$). Now, the quantum PAC (resp. quantum agnostic) learners calls upon an $(\epsilon, \delta)$-PAC (resp. agnostic) \textit{classical} learner to learn on the training set $\{(x_i,y_i)\}_{i=1}^m$, and outputs the resulting \textit{classically learned} hypothesis. Thus, the classical sample complexity sufficiency requirements \cite{daniely2015multiclass, bendavid1995characterizations} continue to hold, and our proof is complete.
\end{proof}

In \cite{daniely2015multiclass}, the classical upper bound
\[
m^{\text{PAC}} = \bigO\Bigg( \frac{\nDim(\cH)(\log(k) + \log(\frac{1}{\epsilon}) + \log(\nDim(\cH))) + \log(\frac{1}{\delta})}{\epsilon} \Bigg),
\]
was shown to hold which has a tighter dependence on $\epsilon$, but a looser dependence on $\nDim(\cH)$. The proof above naturally extends this bound too to the quantum case.

While this paper does not focus on the DS dimension, which was recently shown to characterize multiclass learnability, including the $k \to \infty$ case \cite{brukhim2022characterization}, we note that the upper bound involving the DS dimension can be similarly extended to the quantum setting. The challenges in adapting the proof of the lower bound involving the DS dimension \cite{daniely2014optimal} are discussed in the conclusion (Section \ref{sec:conclusion}).

\section{Classical Online Learning}\label{sec:classical-online-learning}

So far, we have been working with learning in the batch setting, where we are provided with all the examples at once\footnote{This is typical for most settings where we are trying to learn a hypothesis via inductive reasoning (e.g. learning a function to fit data, etc.)}. For several practical applications, it is either impossible to obtain all the examples at once (e.g., recommendation systems), or we simply wish to evolve our learning over time. In these cases, \textit{online learning} \cite{littlestone1988learning} -- where we iteratively improve our hypothesis using examples we receive over time, and using our current hypothesis to predict for the upcoming example -- is the appropriate framework to be placing ourselves in. First, we will introduce known models and results in classical online learning, and a \textit{classical} generalization in Section \ref{model:classical-adversary-provides-a-distribution} that, in turn, provides us with a \textit{quantum} online learning model (Section \ref{sec:quantum_online}) as a natural generalization. For ease of exposition, we begin with a treatment of boolean function classes in the realizable setting.

\subsection{Adversary provides an input} \label{model:classical-adversary-provides-an-input}

Let $\cC := \{c : \cX \to \{0,1\} \}$, and $\cH \subseteq \cC$ (i.e., $\cH \subseteq \{0,1\}^\cX$). A protocol for online learning is a $T$-round procedure described as follows: at the $t$-th round,

\begin{enumerate}[noitemsep]
    \item Adversary provides input point in the domain: $x_t \in \cX$.
    \item Learner uses a hypothesis\footnote{Note that the learner may choose a hypothesis $h_t \in \cC \setminus \cH$, i.e., we do not require the learner to be \textit{proper}.} $h_t \in \cC$, and makes the prediction $\hat{y}_t = h_t(x_t) \in \{0,1\}$.
    \item Adversary provides the input point's label, $y_t = h^\star(x_t)$, where $h^\star \in \cH$.
    \item Learner suffers a loss of 1 (a `mistake'), if $\hat{y}_t \neq y_t$, i.e. $\lossf_\mathbb{I}(h_t,x_t,h^\star) = \indf[h_t(x_t) \neq h^\star(x_t)]$.
\end{enumerate}

Therefore, the learner's total loss is given by,
\begin{equation}\label{eq:indicator_loss}
\lossf_\mathbb{I}(\mathbf{h},\mathbf{x},h^\star) = \sum_{t=1}^T  \lossf_\mathbb{I}(h_t,x_t,h^\star) = \sum_{t=1}^T \indf [h_t(x_t) \neq h^\star(x_t)],
\end{equation}
where we use $\vech := (h_1,\ldots,h_T)$ to denote the sequence of hypotheses that the learner uses, and $\vecx := (x_1,\ldots,x_T)$ to denote the sequence of instances that the adversary provides. The subscript $\mathbb{I}$ (in $\lossf_\mathbb{I}$) indicates that this is the input-based indicator (0-1) loss function\footnote{This is distinct from the probabilistic loss function $\lossf_P$ that we will encounter later in Section~\ref{model:classical-adversary-provides-a-distribution}.}.

The learner chooses an algorithm $\cA$ that will generate the sequence $\vech_\cA$ following the protocol above. The learner's goal is to minimize $\lossf_\mathbb{I}(\vech_\cA,\cH) = \sup_{\vecx, h^\star \in \cH} \Expect[\lossf_\mathbb{I}(\vech_\cA,\vecx,h^\star)]$, i.e., make as few mistakes, on average, as possible regardless of the adversary's (potentially worst-case) choices of sequence of instances $\vecx := (x_1,\ldots,x_T)$ and labeling function $h^\star \in \cH$. For the subsequent bound on $\lossf_\mathbb{I}(\vech_\cA,\cH)$, we first define the combinatorial parameter, Littlestone dimension, $\lDim(\cH)$.

\begin{definition}[Littlestone dimension]\label{def:Ldim}
Let $T$ be a rooted tree whose internal nodes are labeled by elements from $\cX$. Each internal node's left edge and right edge are labeled 0 and 1, respectively. The tree $T$ is $\text{L}$-shattered by $\cH$ if, for every path from root to leaf which traverses the nodes $x_1,\ldots,x_d$, there exists a hypothesis $h \in \cH$ such that, for all $i$, $h(x_i)$ is the label of the edge $(x_i, x_{i+1})$. We define the Littlestone dimension, $\lDim(\cH)$, to be the maximal depth of a complete binary tree that is $\text{L}$-shattered by $\cH$.
\end{definition}

The classical online learning model in this subsection (Section \ref{model:classical-adversary-provides-an-input}) has been thoroughly studied, and the following theorem characterizes it in terms of the Littlestone dimension.

\begin{theorem}[Bounds on $\lossf_\mathbb{I}(\vech_\cA,\cH)$ for the canonical classical online model; Corollary 21.8 in \cite{shalev2014understanding} and Theorem 24 in \cite{daniely2015multiclass}]\label{thm:mistake_bound_ldim}
Let $\cH \subseteq \{0,1\}^\cX$ be a hypothesis class. The Standard Optimal Algorithm (SOA)\footnote{At round $t$, given input $x_t$, the SOA predicts $\hat{y_t} \in \{0,1\}$ that maximizes the Littlestone dimension of the version space consistent with $\hat{y_t}$.} is a deterministic algorithm that achieves a worst-case total loss of $\lDim(\cH)$, i.e. $\lossf_\mathbb{I}(\vech_{SOA},\cH) = \lDim(\cH)$. Furthermore, for any algorithm $\cA$, the expected total loss on the worst-case sequence is at least\footnote{The adversary traverses the shattered tree and provides, at every round, the label that the (randomized) algorithm $\cA$ is less likely to predict.} $\frac{1}{2} \cdot \lDim(\cH)$, i.e. $\lossf_\mathbb{I}(\vech_\cA,\cH) \geq \frac{1}{2} \cdot \lDim(\cH)$.
\end{theorem}

\subsubsection{Can we obtain a quantum generalization of the classical online model?}

Given the popularity and widespread applications of the classical online model, we explore the feasibility of developing a quantum version of the above online model to ultimately inquire whether such a quantum adaptation would be any more powerful from the perspective of the learner and/or the adversary. In essence, as there exists a well-defined ``landscape'' for classical and quantum \textit{batch} learning, we seek to delineate the analogous landscape in the \textit{online} learning context.

To this, if one attempts to na\"ively generalize the above classical model to the quantum setting, an obvious issue arises: the quantum examples of the form (\ref{example:quantum-realizable}) do not split the input-label pair. In particular, an adversary cannot temporally separate its provision of the input point and its label. A first step towards a model that can be generalized to the quantum setting, then, is to reorder the steps at the $t$-th round to $2, 1 \, \& \, 3, 4$ (i.e., where the learner provides a prediction $\hat{y_t}$ after which the adversary presents \textit{both} the input and its label $(x_t,y_t)$). 

While this reordering gives an entirely equivalent model that is, once again, characterized by Littlestone dimension (Theorem \ref{thm:mistake_bound_ldim}), it is not sufficient for a natural quantum generalization. The issue now is that a \textit{classical} adversary only ever presents one (classical) example at each round. How do we go about generalizing a single classical example to a \textit{quantum} adversary's (quantum) example that, in general, sits in superposition? It appears futile to attempt to do so. The missing piece, evidently, is the lack of a notion of a distribution over examples in the classical online model(s) examined so far.

\subsection{Adversary provides a distribution} \label{model:classical-adversary-provides-a-distribution}

Now that we have identified the unfilled gap to transition to the quantum setting, we first state the appropriate \textit{classical} generalization of the canonical classical model (in Section \ref{model:classical-adversary-provides-an-input}) by asking the adversary to, at each $t$, choose a distribution over a set of input-label pairs, from which an explicit input-label pair is then drawn. The protocol for the $T$-round procedure will be as follows: at the $t$-th round,
\begin{enumerate}[noitemsep]
    \item Learner provides a hypothesis $h_t \in \cC$.
    \item Adversary chooses a distribution $D_t: \cX \to [0, 1]$ on the instance space, draws $x \sim D_t$, and reveals $(x,h^\star(x))$ to the learner, where $h^\star \in \cH$.
    \item Learner suffers, but does not ``see'', a loss of $\lossf_P(h_t,D_t,h^\star) := \Prob_{x \sim D_t} (h_t(x) \neq h^\star(x))$.
\end{enumerate}
Here, the learner's total loss is given by,
\begin{equation}\label{eq:prob_loss}
\lossf_P(\mathbf{h},\mathbf{D}, h^\star) = \sum_{t=1}^T  \lossf_P(h_t,D_t,h^\star) = \sum_{t=1}^T \Prob_{x \sim D_t} (h_t(x) \neq h^\star(x)),
\end{equation}
where we additionally use $\mathbf{D} := (D_1,\ldots,D_T)$ to denote the sequence of distributions that the adversary chooses. Analogously, the learner's objective is to choose $\vech$ to minimize $\lossf_P(\mathbf{h},\cH) = \sup_{\vecD, \, h^\star \in \cH} \Expect[\lossf_P(\mathbf{h},\mathbf{D}, h^\star)]$.

We identify this model as the adaptive adversary variant of the online learning model recently considered in Dawid and Tewari \cite{dawid2022learnability}. Note that, if we restrict the adversary, allowing it to choose only point masses, we recover the \textit{reordered} model in Section \ref{model:classical-adversary-provides-an-input}.

From a learning standpoint, the adversary-provides-a-distribution model differs fundamentally from the canonical model in that the learner, here, does not have full information about its own loss at any given round. Since the learner does not know $D_t$, it cannot compute $\lossf_P(h_t,D_t,h^\star)$ for any $t$. In other words, the learner seeks to minimize a quantity that it cannot even compute. This \textit{partial information} setting here, at least at first, appears to be more challenging for the learner as it not only grapples with the inability to compute its loss but also contends with the larger space available to the adversary for its choices ($D_t \in [0,1]^\cX$ vs. $x_t \in \{0,1\}^\cX$).

However, as we will soon illustrate, this perceived challenge proves not to be the case. The key factor influencing this distinction lies in the learner's ability to calculate $\lossf_\mathbb{I}(\vech,\vecx,h^\star)$ for the observed sequence of examples $\vecx$, providing an unbiased estimator for its total loss $\lossf_P(\mathbf{h},\mathbf{D}, h^\star)$. We demonstrate that it is indeed (necessary and) sufficient for a learner in the adversary-provides-a-distribution model to execute SOA on the observed sequence of examples $\vecx$ to achieve a bound analogous to that in the canonical model (cf. Theorem \ref{thm:mistake_bound_ldim}). Before delving into the results, we formally define what a learner, an adversary, and learnability entails for the adversary-provides-a-distribution model we have just discussed.

\begin{definition}[Classical online learner]\label{def:classical-online-learner-adversary-dist} An algorithm $\cA$ is a classical online learner for a hypothesis class $\cH \subseteq \cC$ if having received a sequence of examples over the first $t$ rounds,\\
$(x_1,h^\star(x_1)),\ldots,(x_t,h^\star(x_t))$ where $x_i \sim D_i$ with $D_i$ arbitrary (unknown), $\cA$ outputs a hypothesis $h_{t+1} \in \cC$ at round\footnote{Prior to receiving any examples, $\cA$ outputs some arbitrary hypothesis $h_1 \in \cC$ at round 1.} $t+1$.
\end{definition}

\begin{definition}[Classical adversary]\label{def:classical-online-adversary-adversary-dist} Having received a sequence of hypothesis $\vech|_t = (h_1,\ldots,h_t)$ from the learner, and a sequence of examples $\vecx|_t = (x_t,\ldots,x_t)$ drawn previously from its own prior choices of distributions $\vecD|_t = (D_1,\ldots,D_t)$ over the first $t$ rounds, at round $t+1$, a classical (online) adversary chooses a distribution $D_{t+1}: \cX \to [0,1]$ on the instance space, draws $x_{t+1} \in D_{t+1}$ and reveals $(x_{t+1},h^\star(x_{t+1}))$ to the learner, where $h^\star \in \cH$ is consistent with all preceding labeled examples.
\end{definition}

\begin{definition}[Classical online learnability]\label{def:classical-online-learnability-adversary-dist}
A hypothesis class $\cH$ is classical online learnable if there exists a classical online learning algorithm $\cA$ such that:\newline $\lossf_P(\vech_\cA,\cH) = \sup_{\vecD, \, h^\star \in \cH} \Expect[\lossf_P(\vech_\cA,\mathbf{D}, h^\star)] = o(T)$.
\end{definition}

With these definitions in place, our next objective is to characterize learnability in the adversary-provides-a-distribution framework. When we later introduce the quantum online learning model (Section \ref{sec:quantum_online}), these classical insights will serve as a foundation, enabling us to draw direct comparisons and understand the strong links connecting the \textit{classical} adversary-provides-a-distribution model to the \textit{quantum} online learning setup.

\begin{theorem}[Upper bound on the expected loss for the classical adversary-provides-a-distribution model] \label{thm:loss-upper-bdd-dist}
Let $\cH \subseteq \{0,1\}^\cX$ be a hypothesis class, and $h^\star \in \cH$. For every adversary, there exists a classical online learner for $\cH$ that satisfies
\[
\Expect[\lossf_P(\mathbf{h}, \vecD, h^\star)] = \bigO(\lDim(\cH)).
\]
\end{theorem}

\begin{proof}
Let $\mathbf{D}$ and $h^\star \in \cH$ be arbitrarily chosen. We proceed by first obtaining a high-probability bound for $\lossf_P(\mathbf{h},\mathbf{D},h^\star)$, and then converting it to an in-expectation one. To obtain the high-probability bound, we begin by establishing that the difference between $\lossf_P(\mathbf{h},\mathbf{D},h^\star)$ and $\lossf_\mathbb{I}(\mathbf{h},\vecx, h^\star)$ (see Section \ref{model:classical-adversary-provides-an-input}, (\ref{eq:indicator_loss})) on the revealed stream of examples $\vecx = (x_t)_{t=1}^T$ (with each $x_t \sim D_t$) is the sum of a martingale difference sequence.

Let $M_t := \underbrace{\lossf_P(h_t,D_t,h^\star)}_{P_t} - \underbrace{\lossf_\mathbb{I}(h_t,x_t,h^\star)}_{I_t}$, where $x_t \sim D_t$. With the filtration $\mathbb{F} := (\cF_t)_{t=1}^T$, where $\cF_t$ corresponds to the information revealed\footnote{Note that this is \underline{not} alluding to the information revealed to the learner. Instead, we can think of this information as having been revealed to an arbiter until the end of round $t$, where during each round the arbiter performs the draw $x_t \sim D_t$ on the adversary's communicated choice of $D_t$ and provides $(x_t,h^\star(x_t))$ to the learner.} up to (and, including) round $t$, namely\footnote{Recall, $\mathbf{v}|_t = (v_1,\ldots,v_t)$, i.e. $\mathbf{v}$ restricted to the first $t$ rounds.} $\vech|_t$, $\vecD|_t$, $\vecx|_t$ and $\mathbf{y}|_t := (h^\star(x_i))_{i=1}^t$, we note that $\mathbf{M} := (M_t)_{t=1}^T$ is adapted to $\mathbb{F}$ and $\forall t$,
\begin{align*}
    \Expect[M_t] &= \Expect[P_t] - \Expect[I_t] < \infty \\
    \Expect[M_t|\cF_{t-1}] &= \Expect[P_t|\cF_{t-1}] - \Expect[I_t|\cF_{t-1}] = P_t - P_t = 0,
\end{align*}
where the first line is due to the boundedness ($0 \leq P_t, I_t \leq 1, \; \forall t$) of $P_t$ and $I_t$. And, the second line is due to $\Expect[I_t|\cF_{t-1}] = 1 \cdot \Prob_{x_t \sim D_t} (h_t(x_t) \neq h^\star(x_t)) + 0 \cdot \Prob_{x_t \sim D_t} (h_t(x_t) = h^\star(x_t)) = P_t$ and $\Expect[P_t|\cF_{t-1}] = P_t$ (see Remark \ref{rem:remark-on-randomness}). Therefore, $\mathbf{M}$ is a martingale difference sequence.

We first bound the \textit{predictable quadratic variation} $\langle M_T \rangle$ of $\mathbf{M}$, as follows:
\begin{align} \label{eq:bound-on-wt}
\langle M_T \rangle &:= \sum_{t=1}^T \Expect[M_t^2|\cF_{t-1}] = \sum_{t=1}^T \Expect[(P_t-I_t)^2|\cF_{t-1}] \nonumber\\
& = \sum_{t=1}^T \Big(\Expect[P_t^2|\cF_{t-1}] -2 \Expect[P_tI_t|\cF_{t-1}] + \Expect[I_t^2|\cF_{t-1}]\Big) \nonumber\\
& = \sum_{t=1}^T (P_t^2 - 2P_t^2 + P_t) = \sum_{t=1}^T (P_t - P_t^2) \nonumber\\
& \leq \sum_{t=1}^T P_t,
\end{align}
where the second line is due to linearity of expectation, the third line uses $\Expect[P_tI_t|\cF_{t-1}] = P_t\Expect[I_t|\cF_{t-1}] = P_t \cdot P_t = P_t^2$, and the fourth line is due to $P_t^2 \geq 0, \; \forall t$.
Now, by Theorem 1 of \cite{beygelzimer2011contextual}, with probability $1-\delta$ (for any $\delta > 0$), we have
\begin{align*}
\sum_{t=1}^T M_t &\leq \log \Big( \frac{1}{\delta} \Big) + (e-2)\langle M_T \rangle \iff
\sum_{t=1}^T P_t \leq \sum_{t=1}^T I_t + \log \Big( \frac{1}{\delta} \Big) + (e-2)\langle M_T \rangle.
\end{align*}

Note here that $\sum_{t=1}^T P_t = \lossf_P(\mathbf{h},\mathbf{D}, h^\star)$. Therefore, appealing to Theorem \ref{thm:mistake_bound_ldim} and the inequality on $\langle M_T \rangle$ in (\ref{eq:bound-on-wt}), we have that (with probability $1-\delta$),
\begin{align*}
\lossf_P(\mathbf{h},\mathbf{D}, h^\star) & \leq \lDim(\cH) + \log \Big( \frac{1}{\delta} \Big) + (e-2)\lossf_P(\mathbf{h},\mathbf{D}, h^\star)\\
\implies \lossf_P(\mathbf{h},\mathbf{D}, h^\star) & \leq \frac{1}{1-(e-2)} \lDimH + \frac{1}{1-(e-2)} \log \Big( \frac{1}{\delta} \Big),
\end{align*}
or equivalently, for any $\delta > 0$,
\begin{equation}\label{eq:high-probability-bound-freedman}
\Prob\Big[\lossf_P(\mathbf{h},\mathbf{D}, h^\star) > \frac{1}{1-(e-2)} \cdot \lDimH + \frac{\delta}{1-(e-2)}\Big] \leq e^{-\delta}.
\end{equation}
Now, we compute the in-expectation bound (i.e. a bound on $\Expect[\lossf_P(\mathbf{h},\mathbf{D}, h^\star)]$) guaranteed by the above tail bound (\ref{eq:high-probability-bound-freedman}). Let $\displaystyle c := \frac{1}{1-(e-2)}$.
\begin{align*}
    \Expect[\lossf_P(\mathbf{h},\mathbf{D}, h^\star)] & = \int_0^\infty \Prob[\lossf_P(\mathbf{h},\mathbf{D}, h^\star) > \ell] \, d \ell \\
    & = \int_0^{c \cdot \lDimH} \Prob[\lossf_P(\mathbf{h},\mathbf{D}, h^\star) > \ell] \, d \ell + \int_{c \cdot \lDimH}^\infty \Prob[\lossf_P(\mathbf{h},\mathbf{D}, h^\star) > \ell] \, d \ell \\
   & = \int_0^{c \cdot \lDimH} \Prob[\lossf_P(\mathbf{h},\mathbf{D}, h^\star) > \ell] \, d \ell + c \int_0^\infty \Prob\Big[\lossf_P(\mathbf{h},\mathbf{D}, h^\star) > c \cdot \lDimH + c \delta \Big] \, d \delta \\
   & \leq \int_0^{c \cdot \lDimH} 1 \, d \ell + c \int_0^\infty e^{-\delta} \, d \delta \\
    & = c \cdot \lDimH + c\\
    & = \bigO(\lDim(\cH)),
\end{align*}
where the first line holds as $\lossf_P(\mathbf{h},\mathbf{D}, h^\star) \geq 0$, the third line uses the change of variable $\ell = c \cdot \lDimH + c \delta$, the fourth line uses a na\"ive bound of $\Prob[\lossf_P(\mathbf{h},\mathbf{D}, h^\star) > \ell] \leq 1$ (for $0 \leq \ell \leq c \cdot \lDimH$) on the first integral, and our high-probability bound in (\ref{eq:high-probability-bound-freedman}) on the second integral.
\end{proof}

\begin{remark}[Internal Randomness]\label{rem:remark-on-randomness}
In the proof of Theorem \ref{thm:loss-upper-bdd-dist} (and in subsequent proofs throughout the paper), the filtration $\mathbb{F}$ does not include any information on the internal randomness of the learner or the (adaptive) adversary at any round. In other words, the bound in Theorem \ref{thm:loss-upper-bdd-dist} applies to a $T$-round procedure that is behaving deterministically (in as much as it applies to the learner and the adversary). Crucially, however, it applies to \underline{\textit{all}} such procedures. Therefore, the bounds presented in Theorem \ref{thm:loss-upper-bdd-dist} (as well as others throughout the paper) imply identical bounds on a $T$-round procedure where internal randomness is allowed for both the learner and the adversary. 
\end{remark}

To summarize, Theorem $\ref{thm:loss-upper-bdd-dist}$ tells us that $\lDimH$ continues to be a sufficient condition for learnability under the new (adversary-provides-a-distribution) model in Section \ref{model:classical-adversary-provides-a-distribution}. In other words, a learner that performs SOA on the observed sequence of examples $\vecx$ only suffers a constant overhead under the adversary-provides-a-distribution model as compared to SOA under the canonical online model (Section \ref{model:classical-adversary-provides-an-input}). Next, we show (Theorem \ref{thm:loss-lower-bdd-dist}) that $\lDimH$ is also a necessary condition for learnability, and thus fully characterizes the learnability of the adversary-provides-a-distribution model.

\begin{theorem}[Lower bound on the expected loss for the classical adversary-provides-a-distribution model] \label{thm:loss-lower-bdd-dist}
Let $\cH \subseteq \{0,1\}^\cX$ be a hypothesis class, and $h^\star \in \cH$. For every classical online learner of $\cH$, there exists an adversary such that
\[
\Expect[\lossf_P(\mathbf{h}, \vecD, h^\star)] = \Omega(\lDim(\cH)).
\]
\end{theorem}
\begin{proof}
Consider an adversary which chooses each $D_t$ to be a point mass on the instance space, i.e. the adversary simply chooses an instance $x_t$ at each $t$. Since each $x_t \sim D_t$ is deterministic, we have
\[
\lossf_P(\mathbf{h}, \vecD, h^\star) = \lossf_\mathbb{I}(\mathbf{h}, \vecx, h^\star) = \Omega(\lDim(\cH)),
\]
where the second equality is due to (the lower bound part of) Theorem \ref{thm:mistake_bound_ldim}. By taking expectations, we conclude our proof.
\end{proof}

Now that we have characterized the learnability of the adversary-provides-a-distribution model in Section \ref{model:classical-adversary-provides-a-distribution} (via Theorems \ref{thm:loss-upper-bdd-dist} and \ref{thm:loss-lower-bdd-dist}), we proceed to introduce and present results for related classical online models which arise from successively relaxing, first, the \textit{realizability} assumption and then, the \textit{boolean function class} assumption. These serve to define our scope and lay the groundwork for a comprehensive understanding before introducing the anticipated quantum generalization.

\subsection{Adversary provides a distribution in the agnostic setting} \label{model:classical-adversary-provides-a-distribution-agnostic}
In the agnostic framework, we dispense with the realizability assumption that $h^\star \in \cH$ (i.e., the labels need not be consistent with any hypothesis in the hypothesis class). In fact, the labels need not arise from a labeling function at all, i.e., the examples may be inconsistent\footnote{That is, it is entirely possible to encounter both $(x,0)$ and $(x,1)$ in the sequence of examples.}. The agnostic generalization of the adversary-provides-a-distribution model in Section \ref{model:classical-adversary-provides-a-distribution} is given by the following protocol for the $T$-round procedure: at the $t$-th round,
\begin{enumerate}[noitemsep]
    \item Learner provides a hypothesis $h_t \in \cC$.
    \item Adversary chooses a distribution $D_t: \{0,1\}^{n+1} \to [0, 1]$ on the instance space, draws and reveals $z_t = (x_t,y_t) \sim D_t$ to the learner.
    \item Learner suffers, but does not ``see'', a loss of $\lossf_P(h_t,D_t) := \Prob_{(x,y) \sim D_t} (h_t(x) \neq y)$.
\end{enumerate}

As there may be no hypothesis that provides the true label on every instance over the $T$ rounds, we resort to comparing the learner to the best fixed hypothesis in $\cH$ in hindsight. In other words, the learner's total \textit{regret} is given by,
\begin{align}\label{eq:regret-loss}
\regretf_P(\mathbf{h},\vecD,\cH) = \sum_{t=1}^T \cL_P(h_t,D_t) - \inf_{h \in \cH} \sum_{t=1}^T \cL_P(h,D_t).
\end{align}

Definitions \ref{def:classical-online-learner-adversary-dist}, \ref{def:classical-online-adversary-adversary-dist}, and \ref{def:classical-online-learnability-adversary-dist} are adapted analogously to give us the notions of a classical online learner, adversary and learnability in the agnostic setting. Additionally, we introduce further notation and definitions to facilitate the subsequent theorem.

\vspace*{2ex}

An \textbf{$\cX$-valued tree $\underline{X}$ of depth $T$} is a rooted complete binary tree with nodes labeled by elements of $\cX$. We identify the tree $\underline{X}$ with the sequence $x^{(1)},\ldots,x^{(T)}$ of labeling functions $x^{(i)}: \{\pm 1\}^{i-1} \to \cX$ which provide the labels for each node. Here, $x^{(1)}$ labels the root of the tree, and $x^{(i)}$ for $i > 1$ labels the node obtained by following the path of length $i-1$ from the root, with $+1$ indicating `right' and $-1$ indicating `left'. A path of length $T$ is given by the sequence $\epsilon = (\epsilon_1,\ldots,\epsilon_T) \in \{\pm 1\}^T$. We denote the label at round $t$ along this path as $x^{(t)}(\epsilon)$, understanding that $x^{(t)}$ depends only on the prefix $(\epsilon_1,\ldots,\epsilon_{t-1})$ of $\epsilon$. With this notion of a tree, we define the \textit{sequential Rademacher complexity} of a hypothesis class, $\cH \subseteq \cY^\cX$ \cite{rakhlin2015sequential}.

\begin{definition}[Sequential Rademacher complexity of $\cH$]\label{def:seq-rad-H}
    The \textit{sequential Rademacher complexity} of a function class $\cH \subseteq \reals^\cX$ on an $\cX$-valued tree $\underline{X}$ is defined as
    \[
    \mathfrak{R}_T(\cH,\underline{X}) = \Expect\Big[\sup_{h \in \cH} \frac{1}{T} \sum_{t=1}^T \epsilon_t h(x^{(t)}(\mathbf{\epsilon}))\Big],
    \]
    and
    \[
    \mathfrak{R}_T(\cH) = \sup_{\underline{X}} \mathfrak{R}_T(\cH,\underline{X}),
    \]
    where the outer supremum is taken over all $\cX$-valued trees of depth $T$; $\mathbf{\epsilon} = (\epsilon_1,\ldots,\epsilon_T) \in \{\pm 1\}^T$ is a sequence of i.i.d. Rademacher random variables.
\end{definition}

\begin{definition}[The loss class $\lossf_{\mathbb{I}} \circ \cH$] \label{def:loss-class}
    The loss class, $\lossf_{\mathbb{I}} \circ \cH$, is a boolean hypothesis class given by
    \[
    \lossf_{\mathbb{I}} \circ \cH = \{ l_h : (x,y) \mapsto \indf[h(x) \neq y] \, | \, h \in \cH \}.
    \]
\end{definition}

Note that $\lossf_{\mathbb{I}} \circ \cH \subseteq \{0,1\}^{\cX \times \cY}$, and that its sequential Rademacher complexity is defined analogously to Definition \ref{def:seq-rad-H}. With these definitions in place, we can proceed to present the theorems for the bounds on expected regret for the adversary-provides-a-distribution model in the agnostic setting.

\begin{theorem}[Upper bound on the expected regret for the classical adversary-provides-a-distribution model in the agnostic setting] \label{thm:regret-upper-bdd-dist}
Let $\cH \subseteq \{0,1\}^\cX$ be a hypothesis class. For every adversary, there exists a classical online learner for $\cH$ that satisfies
\[
\Expect[\regretf_P(\mathbf{h},\vecD,\cH)] = \bigO(\sqrt{\lDim(\cH)\cdot T}).
\]
\end{theorem}

\begin{proof}
Let $\vecD$ be an arbitrary sequence of distributions, and let $\vecz = (z_1,\ldots,z_t)$ be a sequence of instances such that $z_i \sim D_i$. Then, defining\footnote{$\regretf_\mathbb{I}(\mathbf{h},\vecz,\cH)$ is precisely the regret of an algorithm in the agnostic generalization of the canonical (adversary-provides-an-input) classical online model in Section \ref{model:classical-adversary-provides-an-input}.} $\regretf_\mathbb{I}(\mathbf{h},\vecz,\cH) := \sum_{t=1}^T \lossf_\mathbb{I}(h_t,z_t) - \inf_{h \in \cH} \sum_{t=1}^T \lossf_\mathbb{I}(h,z_t)$, we have
\begin{align*}
    \regretf_P(\mathbf{h},\vecD,\cH) &= \regretf_P(\mathbf{h},\vecD,\cH) - \regretf_\mathbb{I}(\mathbf{h},\vecz,\cH) + \regretf_\mathbb{I}(\mathbf{h},\vecz,\cH) \\
    &= \underbrace{\sum_{t=1}^T \cL_P(h_t,D_t) - \sum_{t=1}^T \cL_\mathbb{I}(h_t,z_t)}_{\Delta_1} + \underbrace{\inf_{h \in \cH} \sum_{t=1}^T \cL_\mathbb{I}(h,z_t) - \inf_{h \in \cH} \sum_{t=1}^T \cL_P(h,D_t)}_{\Delta_2} + \regretf_\mathbb{I}(\mathbf{h},\vecz,\cH).
\end{align*}
We proceed by bounding the expected value of $\Delta_1$, $\Delta_2$ and $\regretf_\mathbb{I}(\mathbf{h},\vecz,\cH)$ separately.

Working with $\Delta_1$, we let $M_t := \cL_P(h_t,D_t) - \cL_\mathbb{I}(h_t,z_t)$. With the filtration $\mathbb{F} := (\cF_t)_{t=1}^T$, where $\cF_t$ corresponds to the information revealed up to (and, including) round $t$, namely $\vech|_t$, $\vecD|_t$ and $\vecz|_t$, we note that $\mathbf{M} := (M_t)_{t=1}^T$ is adapted to $\mathbb{F}$ and $\forall t$,
\begin{align*}
    \Expect[M_t] &= \Expect[\cL_P(h_t,D_t)] - \Expect[\cL_\mathbb{I}(h_t,z_t)] < \infty \\
    \Expect[M_t|\cF_{t-1}] &= \Expect[\cL_P(h_t,D_t)|\cF_{t-1}] - \Expect[\cL_\mathbb{I}(h_t,z_t)|\cF_{t-1}] = \cL_P(h_t,D_t) - \cL_P(h_t,D_t) = 0,
\end{align*}
where the first line is due to the boundedness ($0 \leq \cL_P(h_t,D_t), \cL_\mathbb{I}(h_t,z_t) \leq 1, \; \forall t$) of $\cL_P(h_t,D_t)$ and $\cL_\mathbb{I}(h_t,z_t)$. And, the second line is due to $\Expect[\cL_\mathbb{I}(h_t,z_t)|\cF_{t-1}] = 1 \cdot \Prob_{z_t \sim D_t} (h_t(x_t) \neq y_t) + 0 \cdot \Prob_{z_t \sim D_t} (h_t(x_t) = y_t) = \cL_P(h_t,D_t)$.
Therefore, $\mathbf{M}$ is a martingale difference sequence. Now, as $\Delta_1 = \sum_{t=1}^T M_t$, by Azuma-Hoeffding's inequality, since $|M_t| < 1$ for all $t$, we have that
\begin{equation}\label{eq:prob-bound-ub-boolean-4-2}
\Prob[ |\Delta_1| \geq \delta] \leq 2\exp\Big(-\frac{\delta^2}{2T}\Big), \;\; \text{for all} \;\; \delta \in \reals^+, T \in \mathbb{Z}^+.
\end{equation}

This allows us to compute the following bound on $\Expect[\Delta_1]$:
\begin{align*}
    \Expect[\Delta_1] \leq \Expect[|\Delta_1|] & \leq \int_0^\infty \Prob[|\Delta_1| \geq \delta] \, dB \\
    & \leq \int_0^\infty 2\exp\Big(-\frac{\delta^2}{2T}\Big) \, dB = \sqrt{2\pi T} = \bigO(\sqrt{T}),
\end{align*}
where the first line holds as $\Delta_1 < |\Delta_1|$ (for the first inequality) and $|\Delta_1| \geq 0$ (for the second inequality), and the second line uses the bound in (\ref{eq:prob-bound-ub-boolean-4-2}).

Next, working with $\Delta_2$, we obtain the following chain of (in)equalities:
\begin{align*}
    \Delta_2 &= \inf_{h \in \cH} \sum_{t=1}^T \cL_\mathbb{I}(h,z_t) - \inf_{h \in \cH} \sum_{t=1}^T \cL_P(h,D_t) \\
    & = - \sup_{h \in \cH} \sum_{t=1}^T -\cL_\mathbb{I}(h,z_t) + \sup_{h \in \cH} \sum_{t=1}^T -\cL_P(h,D_t)\\
    & = - \sup_{h \in \cH} \sum_{t=1}^T (-\cL_\mathbb{I}(h,z_t) + \cL_P(h,D_t) - \cL_P(h,D_t)) + \sup_{h \in \cH} \sum_{t=1}^T -\cL_P(h,D_t)\\
    &\leq \sup_{h \in \cH} \sum_{t=1}^T (\cL_\mathbb{I}(h,z_t) - \cL_P(h,D_t)),
\end{align*}
where the final inequality is a result of the subadditivity of the supremum; to elaborate, we derive $\sup_h f \leq \sup_h (f-g) + \sup_h g \iff \sup_h f - \sup_h g \leq \sup_h (f - g)$, to which we substitute $f = -\cL_P(h,D_t)$ and $g = -\cL_\mathbb{I}(h,z_t) + \cL_P(h,D_t) - \cL_P(h,D_t)$. So far, taking expectations both sides, we have $\Expect[\Delta_2] \leq \Expect[\sup_{h \in \cH} \sum_{t=1}^T (\cL_\mathbb{I}(h,z_t) - \cL_P(h,D_t))]$. 
Next, we bound the in-expectation quantity on the right-hand side to obtain,
\begin{align*}
\Expect\Big[\sup_{h \in \cH} \sum_{t=1}^T (\cL_\mathbb{I}(h,z_t) - \cL_P(h,D_t))\Big] &= \Expect \Bigg[ \sup_{h \in \cH} \sum_{t=1}^T \Big( \indf[h(x_t) \neq y_t] - \Prob_{(x_t,y_t) \sim D_t}[h(x_t) \neq y_t]  \Big) \Bigg] \\
&= \Expect \Bigg[ \sup_{h \in \cH} \sum_{t=1}^T \Big( \indf[h(x_t) \neq y_t] - \Expect_{(x_t,y_t)\sim D_t}\Big[\indf[h(x_t) \neq y_t]|\cF_{t-1}\Big] \Big) \Bigg]\\
& \leq 2T \cdot \mathfrak{R}_T(\lossf_{\mathbb{I}} \circ \cH)\\
& \leq 2T \cdot \mathfrak{R}_T(\cH)\\
& = \bigO(\sqrt{\lDimH \cdot T}),
\end{align*}
where the second line uses a conditioning on the filtration $\cF_{t-1}$, which corresponds to the information revealed up to (and, including) round $t-1$ (namely $\vech|_{t-1}:= (h_1,\ldots,h_{t-1})$, $\vecD|_{t-1}:=(D_1,\ldots,D_{t-1})$ and $\vecz|_{t-1}:=((x_1,y_1),\ldots,(x_{t-1},y_{t-1}))$), the third line is due to Theorem 2 in \cite{rakhlin2015sequential}\footnote{Theorem 2 in \cite{rakhlin2015sequential}, as stated, provides a bound on $\Expect \Big[ \sup_{h \in \cH} \sum_{t=1}^T \Big(\Expect_{(x_t,y_t)\sim D_t}\Big[\indf[h(x_t) \neq y_t]|\cF_{t-1}\Big] - \indf[h(x_t) \neq y_t] \Big) \Big]$. However, the proof of Lemma 9 in \cite{rakhlin2015sequential} notes its validity even with absolute values around the sum, which subsequently ensures that Theorem 2 in \cite{rakhlin2015sequential} also holds in the same generality. This justifies its use here in the following sense: $\Expect \Big[ \sup_{h \in \cH} \sum_{t=1}^T \Big( \indf[h(x_t) \neq y_t] - \Expect_{(x_t,y_t)\sim D_t}\Big[\indf[h(x_t) \neq y_t]|\cF_{t-1}\Big] \Big) \Big] \leq \Expect \Big[ \sup_{h \in \cH} \Big| \sum_{t=1}^T \Big(\Expect_{(x_t,y_t)\sim D_t}\Big[\indf[h(x_t) \neq y_t]|\cF_{t-1}\Big] - \indf[h(x_t) \neq y_t] \Big) \Big| \Big] \leq 2T \cdot \mathfrak{R}_T(\lossf_{\mathbb{I}} \circ \cH)$.\label{footnote:abs-value-thm2-rakhlin}}, the fourth line is due to Theorem 16 of \cite{dawid2022learnability} (with $\cY = \{0,1\}$) and the last line is from the proof of Theorem 12.1 in \cite{alon2021adversarial}.

Finally, also from Theorem 12.1 of \cite{alon2021adversarial}, we have $\regretf_\mathbb{I}(\mathbf{h},\vecz,\cH) = \bigO(\sqrt{\lDimH \cdot T})$. Putting everything together, using the linearity of expectation, we have,
\begin{align*}
\Expect[\regretf_P(\mathbf{h},\vecD,\cH)] &= \Expect[\Delta_1 + \Delta_2 + \regretf_\mathbb{I}(\mathbf{h},\vecz,\cH)]\\
&= \Expect[\Delta_1] + \Expect[\Delta_2] + \Expect[\regretf_\mathbb{I}(\mathbf{h},\vecz,\cH)]\\
&= \bigO(\sqrt{T}) + \bigO(\sqrt{\lDimH \cdot T}) + \bigO(\sqrt{\lDimH \cdot T})\\
&= \bigO(\sqrt{\lDimH \cdot T}).
\end{align*}
\end{proof}

In summary, our theorem reveals that the optimal classical learner for the canonical (adversary-provides-an-input) agnostic model, when provided with the observed sequence of instances $\mathbf{z}$ in the new protocol, experiences, at most, a constant overhead when assessed under the new (adversary-provides-a-distribution) framework. Upon closer examination, it was critical for the bound on $\Expect[\Delta_2]$ in our proof not to be worse than the bound on $\Expect[\regretf_\mathbb{I}(\mathbf{h},\vecz,\cH)]$. It is noteworthy that the bound on $\Expect[\Delta_2]$ is guaranteed by the rate of online uniform convergence (in the frameworks of the sequential Rademacher complexity \cite{rakhlin2015sequential} and the adversarial (uniform) laws of large numbers \cite{alon2021adversarial}), whereas the bound on $\Expect[\regretf_\mathbb{I}(\mathbf{h},\vecz,\cH)]$ is guaranteed by the rate of canonical (agnostic) online learnability. The equivalence between these two rates, for the boolean function class case, played a pivotal role in establishing our result.

Next, we establish (Theorem \ref{thm:regret-lower-bdd-dist}) a matching lower bound for expected regret within the agnostic adversary-provides-a-distribution framework. This fully characterizes agnostic learnability under the adversary-provides-a-distribution framework. As with all lower bound proofs within this framework, we efficiently conclude the proof statement by considering an adversary that exclusively plays point masses.

\begin{theorem}[Lower bound on the expected regret for the classical adversary-provides-a-distribution model in the agnostic setting] \label{thm:regret-lower-bdd-dist}
Let $\cH \subseteq \{0,1\}^\cX$ be a hypothesis class. For every classical online learner of $\cH$, there exists an adversary such that
\[
\Expect[\regretf_P(\mathbf{h}, \vecD, \cH)] = \Omega(\sqrt{\lDim(\cH)\cdot T}).
\]
\end{theorem}
\begin{proof}
As in Theorem \ref{thm:loss-lower-bdd-dist}, we again consider an adversary which chooses each $D_t$ to be a point mass on the instance space, i.e. the adversary simply chooses an instance $z_t$ at each $t$. Since each $z_t \sim D_t$ is deterministic, we have
\[
\Expect[\regretf_P(\mathbf{h}, \vecD, \cH)] = \Expect[\regretf_\mathbb{I}(\mathbf{h}, \vecz, \cH)] = \Omega(\sqrt{\lDim(\cH)\cdot T}),
\]
where the second equality is due to (the lower bound part of) Theorem 21.10 in \cite{shalev2014understanding}.
\end{proof}

\subsection{Adversary provides a distribution in the multiclass setting}\label{sec:classical-adversary-provides-a-distribution-mc}

In Sections \ref{model:classical-adversary-provides-an-input}, \ref{model:classical-adversary-provides-a-distribution} and \ref{model:classical-adversary-provides-a-distribution-agnostic}, we considered boolean hypothesis classes, i.e. $\cH \subseteq \{0,1\}^\cX$. Here, we consider the adversary-provides-a-distribution models in Sections \ref{model:classical-adversary-provides-a-distribution} and \ref{model:classical-adversary-provides-a-distribution-agnostic} extended to the setting of multiclass learning, i.e. $\cH \subseteq \cC := \{c : \cX \to \cY \}$, with $|\cY| > 2$. As stated earlier, the objective here is to lay the groundwork, with a clear understanding of these models in the classical paradigm, before delving into their quantum generalizations.

\subsubsection{Realizable setting}

The online learning protocol in the realizable setting is identical to that specified in Section \ref{model:classical-adversary-provides-a-distribution}, with the added specification of a multiclass hypothesis (and, concept) class. To express our results in this setting, we first define the combinatorial parameter, multiclass Littlestone dimension ($\mclDimH$), which is a generalization of the Littlestone dimension to the multiclass setting.

\begin{definition}[Multiclass Littlestone dimension]
Let $T$ be a rooted tree whose internal nodes are labeled by elements from $\cX$ and whose edges are labeled by elements from $\cY$, such that the edges from a single parent to its child nodes are each labeled with a different label\footnote{In the binary case (where the only ``different labels'' are 0 and 1), it is not hard to see that the definition reduces to that of the Littlestone dimension (Definition \ref{def:Ldim}).}. The tree $T$ is $\text{mcL}$-shattered by $\cH$ if, for every path from root to leaf which traverses the nodes $x_1,\ldots,x_d$, there exists a hypothesis $h \in \cH$ such that, for all $i$, $h(x_i)$ is the label of the edge $(x_i, x_{i+1})$. We define the multiclass Littlestone dimension, $\mclDim(\cH)$, to be the maximal depth of a complete \textbf{binary} tree that is $\text{mcL}$-shattered by $\cH$.
\end{definition}

\begin{theorem}[Upper bound on the expected loss for the classical adversary-provides-a-distribution model in the multiclass setting] \label{thm:loss-upper-bdd-dist-mc}
Let $\cH \subseteq \cY^\cX, \, \text{with} \; |\cY| = k > 2$, be a hypothesis class, and $h^\star \in \cH$. For every adversary, there exists a classical online learner for $\cH$ that satisfies
\[
\Expect[\lossf_P(\mathbf{h}, \vecD, h^\star)] = \bigO(\mclDim(\cH)).
\]
\end{theorem}

\begin{proof}
We follow the steps in the proof of Theorem \ref{thm:loss-upper-bdd-dist}, which continue to hold in the multiclass setting, with a minor difference: the upper bound on $\sum_{t=1}^T \lossf_\mathbb{I}(h_t,x_t,h^\star)$ is now $\mclDimH$ instead of $\lDimH$. As a result, we arrive at the following high-probability bound: with probability $1-\delta$ (for any $\delta > 0$), $\lossf_P(\mathbf{h},\mathbf{D}, h^\star) \leq \frac{1}{1-(e-2)} \mclDimH + \frac{1}{1-(e-2)} \log( \frac{1}{\delta}),$ or equivalently, the following tail bound: for any $\delta > 0$,
\[
\Prob\Big[\lossf_P(\mathbf{h},\mathbf{D}, h^\star) > \frac{1}{1-(e-2)} \cdot \mclDimH + \frac{\delta}{1-(e-2)}\Big] \leq e^{-\delta},
\]
from which we obtain the desired in-expectation result: $\Expect[\lossf_P(\mathbf{h},\mathbf{D}, h^\star)] = \bigO(\mclDimH)$.
\end{proof}

\begin{theorem}[Lower bound on the expected loss for the classical adversary-provides-a-distribution model in the multiclass setting] \label{thm:loss-lower-bdd-dist-mc}
Let $\cH \subseteq \cY^\cX, \, \text{with} \; |\cY| = k > 2$, be a hypothesis class, and $h^\star \in \cH$. For every classical online learner of $\cH$, there exists an adversary such that
\[
\Expect[\lossf_P(\mathbf{h}, \vecD, h^\star)] = \Omega(\mclDim(\cH)).
\]
\end{theorem}
\begin{proof}
As in the online lower bound proofs thus far, we consider an adversary which chooses each $D_t$ to be a point mass on the instance space, i.e. the adversary simply chooses an instance $x_t$ at each $t$. Since each $x_t \sim D_t$ is deterministic, we have
\[
\Expect[\lossf_P(\mathbf{h}, \vecD, h^\star)] = \Expect[\lossf_\mathbb{I}(\mathbf{h}, \vecx, h^\star)] = \Omega(\mclDim(\cH)),
\]
where the second equality is due to (the lower bound part of) Theorem 24 in \cite{daniely2015multiclass}.
\end{proof}

Theorems \ref{thm:loss-upper-bdd-dist-mc} and \ref{thm:loss-lower-bdd-dist-mc} together imply that $\mclDimH$ continues to characterize multiclass learnability in the realizable setting under the adversary-provides-a-distribution framework. Given that the rate $\Theta(\mclDimH)$ is independent of $k$, $\mclDimH$ characterizes realizable learnability even in the unbounded label space case, mirroring the scenario in the adversary-provides-an-input framework (Theorem 24, \cite{daniely2015multiclass}).

\subsubsection{Agnostic setting}

Here, the online learning protocol is identical to that specified in Section \ref{model:classical-adversary-provides-a-distribution-agnostic}, with the added specification of a multiclass hypothesis (and, concept) class. We introduce specific definitions and lemmas to facilitate the subsequent theorem, which provides an upper bound on the expected regret of an online learner in this multiclass agnostic adversary-provides-a-distribution setting. For notation related to trees, please refer to Section \ref{model:classical-adversary-provides-a-distribution-agnostic} (paragraph preceding Definition \ref{def:seq-rad-H}).

\begin{definition}[0-cover of a hypothesis class $\cH$ on a tree $\underline{X}$] \label{def:0-cover}
    A set $\cV$ of $\cY$-valued trees is a \textit{0-cover} of $\cH \subseteq \cY^\cX$ on an $\cX$-valued tree $\underline{X}$ of depth $T$ if
    \[
    \forall h \in \cH, \forall \epsilon \in \{\pm 1\}^T, \exists \, \underline{V} \in \cV, \; \text{s.t.}, \; v^{(t)}(\epsilon) = h(x^{(t)}(\epsilon)),
    \]
    for all $t \in [T]$.
\end{definition}

\begin{definition}[Covering number of a hypothesis class $\cH$ on a tree $\underline{X}$]\label{def:covering-number}
    The covering number of a hypothesis class $\cH \subseteq \cY^\cX$ on an $\cX$-valued tree $\underline{X}$, $\mathcal{N}(\cH,\underline{X})$, is defined as follows:
    \[
        \mathcal{N}(\cH,\underline{X}) := \min \{|\cV| : \cV \; \text{is a 0-cover of $\cH$ on} \; \underline{X} \},
    \]
    i.e. the size of the smallest set $\cV$ (of trees) that 0-covers $\underline{X}$.
\end{definition}

\begin{lemma}[$\cN(\lossf_{\mathbb{I}} \circ \cH, \underline{Z}) \leq \cN(\cH, \underline{X})$] \label{lem:bound-bw-0-covers-loss-class}
Let $\cH \subseteq \cY^\cX$. Let $\underline{Z}$ be a $(\cX \times \cY)$-valued tree,\footnote{The nodes of $\underline{Z}$ are of the form $z = (x,y)$, where $x \in \cX$ and $y \in \cY$.} and let $\underline{X}$ be the $\cX$-valued tree obtained by extracting the $\cX$-component from each node of $\underline{Z}$. Then,
\[
\cN(\lossf_{\mathbb{I}} \circ \cH, \underline{Z}) \leq \cN(\cH, \underline{X}).
\]
\end{lemma}
\begin{proof}
Let the set $\cV$ be the smallest set of trees that form a 0-cover of $\cH$ on $\underline{X}$. For each tree $\underline{V} \in \cV$, we construct (and add to $\mathcal{W}$) a tree $\underline{W}$, given by
\[
w^{(t)}(\epsilon) = \begin{cases} 0, \quad & \text{if} \; \; v^{(t)}(\epsilon) = y^{(t)}(\epsilon)\\
1, \quad & \text{otherwise} \end{cases},
\]
where $y^{(t)}(\epsilon)$ is the $y$-label of the node of $\underline{Z}$ encountered after having traversed the path $(\epsilon_1,\ldots,\epsilon_{t-1})$. We claim that $\mathcal{W}$ provides a 0-cover of $\lossf_{\mathbb{I}} \circ \cH$ on $\underline{Z}$.

As the set $\cV$ of trees forms a 0-cover of $\cH$ on $\underline{X}$, we know that 
\[
\forall h \in \cH, \forall \epsilon \in \{\pm 1\}^T, \exists \, \underline{V} \in \cV, \; \text{s.t.}, \; v^{(t)}(\epsilon) = h(x^{(t)}(\epsilon)),
\]
for all $t \in [T]$. Therefore, by construction,
\[
\forall h \in \cH, \forall \epsilon \in \{\pm 1\}^T, \exists \, \underline{W} \in \mathcal{W}, \; \text{s.t.}, \; w^{(t)}(\epsilon) = \indf[h(x^{(t)}(\epsilon)) \neq y^{(t)}(\epsilon)],
\]
for all $t \in [T]$. So, we obtain
\[
\cN(\lossf_{\mathbb{I}} \circ \cH, \underline{Z}) \leq |\mathcal{W}| = |\cV| = \cN(\cH, \underline{X}),
\]
completing our proof.
\end{proof}

\begin{lemma} [$\cN(\cH, \underline{X}) \leq (Tk)^{\mclDimH}$] \label{lem:bound-on-0-cover-mc}
Let $\cH \subseteq \cY^\cX$ with $|\cY| = k > 2$, and $\underline{X}$ be an $\cX$-valued tree. Then,
\[
\cN(\cH, \underline{X}) \leq (Tk)^{\mclDimH}.
\]
\end{lemma}

\begin{proof}
Our main idea is that the set of experts, as defined in the proof of Theorem 25 in \cite{daniely2015multiclass} (reproduced below), can be used to construct a 0-cover of $\cH$ on $\underline{X}$.

\begin{quote}
{Given time horizon $T$, let $A_T = \{A \subset [T] \; | \; |A| \leq \mclDimH \}$. For every $A \in A_T$ and $\phi : A \to \cY$, we define an expert $E_{A,\phi}$. The expert $E_{A,\phi}$ imitates the SOA algorithm when it errs exactly on the examples $\{x_t \; | \; t \in A\}$ and the true labels of these examples are determined by $\phi$. Formally, the expert $E_{A,\phi}$ proceeds as follows:
\begin{quote}
{
    Set $V_1 = \cH$.\\
    For $t=1,2,\ldots,T$:\\
    \hspace*{5ex} Receive $x_t$.\\
    \hspace*{5ex} If $t \in A$, set $\hat{y}_t = \phi(t)$.\\
    \hspace*{5ex} If $t \not\in A$, set $\hat{y}_t = \argmax_{y \in \cY} \mclDim(\{h \in V_t: h(x_t) = y\})$.\\
    \hspace*{5ex} Predict $\hat{y}_t$ and update $V_{t+1} = \{h \in V_t : h(x_t) = \hat{y}_t\}$.
}
\end{quote}
}
\end{quote}

Throughout this proof, we use the notation $E_{A,\phi}(x_1, \ldots, x_t)$ to denote the label $\hat{y}_t$ predicted at round $t$ by the expert $E_{A,\phi}$ after processing the sequence of instances $(x_1, \ldots, x_t)$. Similarly, given any true labeling function $h^\star \in \cH$, we use the notation $\mathrm{SOA}^{h^\star}(x_1, \ldots, x_t)$ to denote the label predicted at round $t$ by the \textit{standard} SOA algorithm, after processing the sequence of instances $(x_1, \ldots, x_t)$ and updating its version space based on the labels $h^\star(x_1), \ldots, h^\star(x_{t-1})$.

Now, the set of experts $\mathcal{E} = \{E_{A,\phi}\}$ has size $|\mathcal{E}| = \sum_{j=0}^{\mclDimH} \binom{T}{j} k^j \leq (Tk)^{\mclDimH}$ with the following \textit{expert guarantee}:
\begin{quote}
For any sequence $(x_1, \ldots, x_T)$ of instances, and any $h \in \cH$, there exists an expert $E^\star \in \mathcal{E}$ such that $E^\star(x_1,\ldots,x_t) = h(x_t), \, \forall \, t \in [T]$.
\end{quote}

Indeed, given an arbitrary sequence of instances $(x_1, \ldots, x_T)$ and an arbitrary $h \in \cH$, the expert $E^\star = E_{A^\star, \phi^\star}$ that satisfies the guarantee corresponds to $A^\star := \{t \in [T] : \mathrm{SOA}^h(x_1, \ldots, x_t) \neq h(x_t)\}$ and $\phi^\star: A^\star \to \cY$ defined by $\phi^\star(t) := h(x_t)$. Since the SOA algorithm makes at most $\mclDimH$ mistakes, we have $|A^\star| \leq \mclDimH$, and therefore $E^\star \in \mathcal{E}$ by construction.

\vspace*{2ex}

Next, for each expert $E \in \mathcal{E}$, we add a tree $\underline{V_E}$ to $\mathcal{V}$, given by
\[
\underline{V_E}^{(t)}(\epsilon) = E(x^{(1)}(\epsilon),\ldots,x^{(t)}(\epsilon)), \quad \forall \, t \in [T],
\]
where $\epsilon = (\epsilon_1,\ldots,\epsilon_T) \in \{\pm 1\}^T$ is a sequence of i.i.d. Rademacher random variables. We now verify that the set of trees $\mathcal{V}$ forms a 0-cover of $\cH$ on $\underline{X}$.

\vspace*{2ex}

Fix an arbitrary $h \in \cH$ and an arbitrary $\epsilon = (\epsilon_1,\ldots,\epsilon_T) \in \{\pm 1\}^T$. For the sequence of examples on the path $(x^{(1)}(\epsilon),\ldots,x^{(T)}(\epsilon))$, by the expert guarantee and our construction of $\mathcal{V}$ above, we have that there exists $E^\star \in \mathcal{E}$ such that $\underline{V_{E^\star}}^{(t)}(\epsilon) = E^\star(x^{(1)}(\epsilon),\ldots,x^{(t)}(\epsilon)) = h(x^{(t)}(\epsilon)), \, \forall t \in [T]$. Therefore, by Definition \ref{def:0-cover}, we see that $\mathcal{V}$ forms a 0-cover of $\cH$ on $\underline{X}$. Hence,
\[
\cN(\cH, \underline{X}) = \min \{|\cV| : \cV \; \text{is a 0-cover of $\cH$ on} \; \underline{X} \} \leq |\cV| = |\mathcal{E}| \leq (Tk)^{\mclDimH},
\]
as desired.
\end{proof}

\begin{theorem}[Upper bound on the expected regret for the classical adversary-provides-a-distribution model in the multiclass agnostic setting] \label{thm:regret-upper-bdd-dist-mc}
Let $\cH \subseteq \cY^\cX, \, \text{with} \; |\cY| = k > 2$, be a hypothesis class, and $h^\star \in \cH$. For every adversary, there exists a classical online learner for $\cH$ that satisfies
\[
\Expect[\regretf_P(\mathbf{h}, \vecD, \cH))] = \bigO(\sqrt{\mclDim(\cH) \cdot T\log(Tk)}).
\]
\end{theorem}

\begin{proof}
    We follow the steps in the proof of Theorem \ref{thm:regret-upper-bdd-dist}, which, in a general sense, are applicable in the multiclass setting. However, a key bound used in the proof of Theorem \ref{thm:regret-upper-bdd-dist} ($\mathfrak{R}_T(\lossf_{\mathbb{I}} \circ \cH) \leq \mathfrak{R}_T(\cH)$) is not known to continue to hold in the multiclass setting, forcing us to handle the rest of the proof differently. Some interesting insights follow from this deviation, which is elaborated in the proof below, as well as the discussion that follows.

    We recall the preliminaries. Let $\vecD$ be an arbitrary sequence of distributions, and let $\vecz = (z_1,\ldots,z_t)$ be a sequence of instances such that $z_i \sim D_i$. Then, defining\footnote{$\regretf_\mathbb{I}(\mathbf{h},\vecz,\cH)$ is the regret of an algorithm in the \textit{multiclass agnostic} generalization of the canonical (adversary-provides-an-input) classical online model in Section \ref{model:classical-adversary-provides-an-input}.} $\regretf_\mathbb{I}(\mathbf{h},\vecz,\cH) = \sum_{t=1}^T \lossf_\mathbb{I}(h_t,z_t) - \inf_{h \in \cH} \sum_{t=1}^T \lossf_\mathbb{I}(h,z_t)$, we have
\begin{align*}
    \regretf_P(\mathbf{h},\vecD,\cH) &= \regretf_P(\mathbf{h},\vecD,\cH) - \regretf_\mathbb{I}(\mathbf{h},\vecz,\cH) + \regretf_\mathbb{I}(\mathbf{h},\vecz,\cH) \\
    &= \underbrace{\sum_{t=1}^T \cL_P(h_t,D_t) - \sum_{t=1}^T \cL_\mathbb{I}(h_t,z_t)}_{\Delta_1} + \underbrace{\inf_{h \in \cH} \sum_{t=1}^T \cL_\mathbb{I}(h,z_t) - \inf_{h \in \cH} \sum_{t=1}^T \cL_P(h,D_t)}_{\Delta_2} + \regretf_\mathbb{I}(\mathbf{h},\vecz,\cH).
\end{align*}
We proceed by bounding the expected value of $\Delta_1$, $\Delta_2$ and $\regretf_\mathbb{I}(\mathbf{h},\vecz,\cH)$ separately. First, our bound $\Expect[\Delta_1] \leq \bigO(\sqrt{T})$ using Azuma-Hoeffding's inequality in Theorem \ref{thm:regret-upper-bdd-dist} is independent of the form of $\cH$ and, in particular, continues to hold for a multiclass $\cH$.

Next, with regard to $\Expect[\Delta_2]$, we recover the chain of inequalities in Theorem \ref{thm:regret-upper-bdd-dist} leading up to
\begin{equation} \label{eq:ineq-before-two-lemmas}
\Expect[\Delta_2] \leq 2T \cdot \mathfrak{R}_T(\lossf_{\mathbb{I}} \circ \cH).
\end{equation}
However, the result in Theorem 16 of \cite{dawid2022learnability} ($\mathfrak{R}_T(\lossf_{\mathbb{I}} \circ \cH) \leq \mathfrak{R}_T(\cH)$) only applies when $\cH$ is a boolean hypothesis class. Therefore, we proceed with an explicit bound on $\mathfrak{R}_T(\lossf_{\mathbb{I}} \circ \cH)$ using a covering number argument. Let $\underline{Z}$ be an $(\cX \times \cY)$-valued tree, and let $\underline{X}$ be the $\cX$-valued tree obtained by extracting the $\cX$-component from each node of $\underline{Z}$. We provide a chain of inequalities starting from (\ref{eq:ineq-before-two-lemmas}):
\begin{align*}
    \Expect[\Delta_2] &\leq 2T \cdot \mathfrak{R}_T(\lossf_{\mathbb{I}} \circ \cH)\\
    & \leq \frac{24T\sqrt{\log \cN(\lossf_{\mathbb{I}} \circ \cH, \underline{Z})}}{\sqrt{T}}\\
    & \leq \frac{24T\sqrt{\log \cN(\cH, \underline{X})}}{\sqrt{T}}\\
    & \leq \frac{24T\sqrt{\mclDimH \cdot \log(Tk)}}{\sqrt{T}}\\
    & = \bigO(\sqrt{\mclDim(\cH) \cdot T \log (Tk)}),
\end{align*}
where the second line is from Theorem 3 and Definition 5 of \cite{rakhlin2015sequential}, the third line is from Lemma \ref{lem:bound-bw-0-covers-loss-class}, and the fourth line is from Lemma \ref{lem:bound-on-0-cover-mc}. Finally, from Theorem 4 of \cite{hanneke2023multiclass}, we have $\regretf_\mathbb{I}(\mathbf{h},\vecz,\cH) = \bigO\Big(\sqrt{\mclDimH \cdot T \log\Big( {\frac{T}{\mclDimH}}\Big)}\Big)$, when $T \geq 2 \cdot \mclDimH$. Putting everything together, using the linearity of expectation, we have,
\begin{align*}
\Expect[\regretf_P(\mathbf{h},\vecD,\cH)] &= \Expect[\Delta_1 + \Delta_2 + \regretf_\mathbb{I}(\mathbf{h},\vecz,\cH)]\\
&= \Expect[\Delta_1] + \Expect[\Delta_2] + \Expect[\regretf_\mathbb{I}(\mathbf{h},\vecz,\cH)]\\
&= \bigO(\sqrt{T}) + \bigO(\sqrt{\mclDim(\cH) \cdot T \log (Tk)})\\
& \qquad\qquad\;\; +\bigO\Bigg(\sqrt{\mclDimH \cdot T \log\Big( {\frac{T}{\mclDimH}}\Big)}\Bigg)\\
&= \bigO(\sqrt{\mclDim(\cH) \cdot T \log (Tk)}),
\end{align*}
where the third line holds as point-wise bounds imply bounds in-expectation\footnote{Since the point-wise bound holds for $T \geq 2 \cdot \mclDimH$, it is clear that $\Expect[\regretf_\mathbb{I}(\mathbf{h},\vecz,\cH)] \leq 2 \cdot \mclDimH + \bigO\Big(\sqrt{\mclDimH \cdot T \log\Big( {\frac{T}{\mclDimH}}\Big)}\Big)$, where the latter term dominates when $T \geq 4 \cdot \mclDimH$.}, and the last line holds as $k > 2$ and $\mclDimH \geq 1 \implies k \geq \frac{1}{\mclDimH}$.
\end{proof}

In summary, our theorem provides a $k$-dependent upper bound on $\Expect[\regretf_P(\mathbf{h},\vecD,\cH)]$, while the corresponding lower bound (presented next, Theorem \ref{thm:regret-lower-bdd-dist-mc}), is $k$-independent. Although the optimal classical learner for the canonical multiclass agnostic model bridges this gap, as indicated by the $k$-independent bound used on $\Expect[\regretf_\mathbb{I}(\mathbf{h},\vecz,\cH)]$, it remains unclear whether we can establish a $k$-\textit{independent} upper bound on $\Expect[\regretf_P(\mathbf{h},\vecD,\cH)]$. Our current proof strategy\footnote{
Currently, we employ the optimal classical learner for the canonical multiclass agnostic model by presenting it with the observed sequence of instances $\vecz$ in the adversary-provides-a-distribution protocol and evaluate its performance under the adversary-provides-a-distribution framework.} faces challenges in achieving this goal due to the following observation. The bound on $\Expect[\Delta_2]$ is guaranteed by the rate of online uniform convergence of the loss ($\lossf_\mathbb{I} \circ \cH$) class (in the frameworks of sequential Rademacher complexity \cite{rakhlin2015sequential} and adversarial (uniform) laws of large numbers \cite{alon2021adversarial}). Meanwhile, the bound on $\Expect[\regretf_\mathbb{I}(\mathbf{h},\vecz,\cH)]$ is guaranteed by the rate of canonical (agnostic) online learnability of $\cH$. However, in the multiclass function class case, as demonstrated by Theorem 7 and Example 1 in \cite{hanneke2023multiclass}, these two rates are \textit{not} equivalent.

\begin{theorem}[Lower bound on the expected regret for the classical adversary-provides-a-distribution model in the multiclass agnostic setting] \label{thm:regret-lower-bdd-dist-mc}
Let $\cH \subseteq \cY^\cX, \, \text{with} \; |\cY| = k > 2$, be a hypothesis class. For every classical online learner of $\cH$, there exists an adversary such that
\[
\Expect[\regretf_P(\mathbf{h}, \vecD, \cH)] = \Omega(\sqrt{\mclDim(\cH)\cdot T}).
\]
\end{theorem}
\begin{proof}
As in the online lower bound proofs thus far, we consider an adversary which chooses each $D_t$ to be a point mass on the instance space, i.e., the adversary simply chooses an instance $z_t$ at each $t$. Since each $z_t \sim D_t$ is deterministic, we have
\[
\Expect[\regretf_P(\mathbf{h}, \vecD, \cH)] = \Expect[\regretf_\mathbb{I}(\mathbf{h}, \vecz, \cH)] = \Omega(\sqrt{\mclDim(\cH)\cdot T}),
\]
where the second equality is due to Theorem 26 in \cite{daniely2015multiclass}.
\end{proof}

Due to Theorems \ref{thm:regret-upper-bdd-dist-mc} and \ref{thm:regret-lower-bdd-dist-mc}, we have characterized multiclass agnostic learnability in the adversary-provides-a-distribution setting for the \textit{bounded} label space ($k < \infty$) case.

\section{Quantum Online Learning} \label{sec:quantum_online}

Equipped with our models in Sections \ref{model:classical-adversary-provides-a-distribution} and \ref{model:classical-adversary-provides-a-distribution-agnostic}, we are finally ready to introduce our quantum online learning model. However, prior to the model description, we clarify our scope. In its nascent existence, quantum online learning has primarily focused on the online learning of \textit{quantum states} \cite{aaronson2018online, quek2021private, anshu2024survey}. In contrast, our focus in this paper is on the online learning of \textit{classical functions} via quantum examples. Our scope is motivated by the abundance of classical online learning literature \cite{littlestone1988learning, ben2009agnostic, daniely2015multiclass, shalev2014understanding}, as well as our results in Sections \ref{model:classical-adversary-provides-a-distribution}, \ref{model:classical-adversary-provides-a-distribution-agnostic} and \ref{sec:classical-adversary-provides-a-distribution-mc}, that presents us with an at-the-ready comparison.

\subsection{Model Description} \label{sec:quantum-online-model}

Let $\cH \subseteq \cY^{\cX}$ be a hypothesis class. Identifying the $T$-round protocol in Section \ref{model:classical-adversary-provides-a-distribution} (corresp. Section \ref{model:classical-adversary-provides-a-distribution-agnostic}) with the definition of a quantum example in (\ref{example:quantum-realizable}) (corresp. (\ref{example:quantum-agnostic})), we obtain the following ``natural'' model for quantum online learning. The $T$-round protocol proceeds as follows: at the $t$-th round,
\begin{enumerate}[noitemsep]
    \item Learner provides a hypothesis $h_t: \cX \to \cY$.
    \item Adversary reveals an example $\ket{\psi_t}$ where
    \begin{enumerate}[noitemsep]
        \item $\ket{\psi_t} = \sum_{x \in \cX} \sqrt{D_t(x)} \ket{x,h^\star(x)}$ for some $D_t : \cX \to [0,1]$ and $h^\star \in \cH$ (realizable),
        \item $\ket{\psi_t} = \sum_{x \in \cX, \, y \in \cY} \sqrt{D_t(x,y)} \ket{x,y}$ for some $D_t : \cX \times \cY \to [0,1]$ (agnostic\footnote{As in the classical case, the adversary need not be consistent: i.e., they could reveal, for e.g., both $\ket{x,0}$ and $\ket{x,1}$ during the $T$-round protocol.}).
    \end{enumerate}
    \item Learner incurs loss\footnote{As an aside, for those who favor a mistake model, it is possible to define it by specifying a threshold $\epsilon$. In this case, a mistake occurs in a round iff $\lossf_P > \epsilon$, i.e. $\lossf_\mathbb{I}^\epsilon = \indf[\lossf_P > \epsilon]$. We do not investigate this mistake model.} 
    \begin{enumerate}[noitemsep]
        \item $\lossf_P(h_t,D_t,h^\star) := \Prob_{x \sim D_t} (h_t(x) \neq h^\star(x))$ (realizable),
        \item $\lossf_P(h_t,D_t) := \Prob_{(x,y) \sim D_t} (h_t(x) \neq y)$ (agnostic).
    \end{enumerate}
\end{enumerate}
As in the classical adversary-provides-a-distribution models, the learner's total loss in the realizable case continues to be given by $\lossf_P(\vech,\vecD,h^\star) = \sum_{t=1}^T \lossf_P(h_t,D_t,h^\star)$ (ref. (\ref{eq:prob_loss})), while in the agnostic case, the learner's total regret continues to be expressed as $\regretf_P(\vech,\vecD,\cH) = \sum_{t=1}^T \lossf_P(h_t,D_t) - \inf_{h \in \cH} \sum_{t=1}^T  \lossf_P(h,D_t)$ (ref. (\ref{eq:regret-loss})).

For this model, we formally define what a quantum online learner, a quantum adversary, and quantum online learnability entails. While these definitions are analogous to Definitions \ref{def:classical-online-learner-adversary-dist}, \ref{def:classical-online-adversary-adversary-dist}, and \ref{def:classical-online-learnability-adversary-dist}, we present them here for the sake of completeness.

\begin{definition}[Quantum online learner]\label{def:quantum-online-learner} An algorithm $\cA$ is a quantum online learner for a hypothesis class $\cH$ if having received a sequence of quantum examples $(\ket{\psi_i})_{i=1}^t$ (of the form in 2. (a) or 2. (b) of Section \ref{sec:quantum-online-model}) over the first $t$ rounds, $\cA$ outputs a hypothesis $h_{t+1}:\cX \to \cY$ at round\footnote{Prior to receiving any examples, $\cA$ outputs some arbitrary hypothesis $h_1 : \cX \to \cY$ at round 1.} $t+1$.
\end{definition}

\begin{definition}[Quantum adversary]\label{def:quantum-online-adversary} Having received a sequence of hypothesis $\vech|_t = (h_1,\ldots,h_t)$ from the learner, and with knowledge of its own prior choices of quantum examples, $(\ket{\psi_i})_{i=1}^t$, over the first $t$ rounds, at round $t+1$, a quantum (online) adversary chooses a distribution $D_{t+1}$ (on $\cX$ (realizable) or on $\cX \times \cY$ (agnostic)) and discloses the corresponding quantum example $\ket{\psi_{t+1}}$ (with consistent labeling throughout the protocol in the realizable case) to the learner.
\end{definition}

\begin{definition}[Quantum online learnability]\label{def:quantum-online-learnability}
A hypothesis class $\cH$ is quantum online learnable if there exists a quantum online learning algorithm $\cA$ such that
\begin{itemize}[noitemsep,nolistsep]
    \item $\lossf_P(\vech_\cA,\cH) = \sup_{\vecD, \, h^\star \in \cH} \Expect[\lossf_P(\vech_\cA,\mathbf{D}, h^\star)] = o(T)$ (realizable),
    \item $\regretf_P(\vech_\cA,\cH) = \sup_{\vecD} \Expect[\regretf_P(\vech_\cA,\vecD,\cH)] = o(T)$ (agnostic).
\end{itemize}
\end{definition}

\subsection{Binary Classification}

Under the quantum online learning model described above in Section \ref{sec:quantum-online-model}, we bound the expected regret (in both the realizable and the agnostic cases) of a quantum online learner for a boolean hypothesis class.

\begin{theorem}[Lower bounds on expected loss/regret for quantum online binary classification] \label{thm:quantum-online-boolean-lb}
Let $\cH \subseteq \{0,1\}^\cX$, be a hypothesis class, and $h^\star \in \cH$. For every quantum  online learner of $\cH$, there exists a quantum adversary such that
\begin{align*}
    \Expect[\lossf_P(\vech, \vecD, h^\star)] & = \Omega(\lDim(\cH)) \; \quad \text{(realizable), and} \\
    \Expect[\regretf_P(\vech, \vecD, \cH)] & = \Omega(\sqrt{\lDim(\cH) \cdot T}) \; \quad \text{(agnostic).}
\end{align*}
\end{theorem}

\begin{proof} Let $\cA_Q$ be an arbitrary, but fixed, \textit{quantum} online learning algorithm for $\cH$. We proceed using a reduction argument. To do this, we examine the scenario where a classical adversary chooses $D_t$ to be a point mass for each $t$ (i.e. the adversary simply chooses an instance $x_t$ (realizable) or $z_t = (x_t,y_t)$ (agnostic) at each $t$), and analyze the loss/regret bound for the following classical learner $\cA_C$ that accesses $\cA_Q$ as a ``black box''. At the $t$-th round,
\begin{enumerate}[noitemsep]
\item $\cA_C$ provides hypothesis $h_t^Q$ (received from $\cA_Q$ in the previous round).
\item Adversary reveals $(x_t,y_t)$ to $\cA_C$ (in the realizable case, $y_t = h^\star(x_t)$ for some $h^\star \in \cH$).
\item $\cA_C$ state prepares $\ket{\psi_t} = \ket{x_t, y_t}$ and passes it as input to $\cA_Q$.
\item $\cA_Q$ outputs hypothesis $h_{t+1}^Q$ to $\cA_C$.
\end{enumerate}

We provide a bound first for the realizable case. Since $\cA_C$ plays $h_t^Q$ at each $t$, it is clear, for our setup, that $\lossf_P(\vech_{\cA_C}, \vecD, h^\star) = \lossf_\mathbb{I}(\vech_{\cA_C}, \vecx, h^\star) \leq \lossf_\mathbb{I}(\vech_{\cA_Q}, \vecx, h^\star) = \lossf_P(\vech_{\cA_Q}, \vecD, h^\star)$. Taking expectations, and noting that $\Expect[\lossf_P(\vech_{\cA_C}, \vecD, h^\star)] = \Omega(\lDimH)$ (from Theorem \ref{thm:loss-lower-bdd-dist}), we have shown $\Expect[\lossf_P(\vech_{\cA_Q}, \vecD, h^\star)] = \Omega(\lDimH)$. Since, $\cA_Q$ was chosen arbitrarily, we deduce that $\Expect[\lossf_P(\vech, \vecD, h^\star)] = \Omega(\lDim(\cH))$.

The agnostic case follows an identical argument; we obtain the following chain of (in)equalities, $\regretf_P(\vech_{\cA_C}, \vecD, \cH) = \regretf_\mathbb{I}(\vech_{\cA_C}, \vecz, \cH) \leq \regretf_\mathbb{I}(\vech_{\cA_Q}, \vecz, \cH) = \regretf_P(\vech_{\cA_Q}, \vecD, \cH)$. Taking expectations, and noting $\Expect[\regretf_P(\vech_{\cA_C}, \vecD, \cH)] = \Omega(\sqrt{\lDimH \cdot T})$ (from Theorem \ref{thm:regret-lower-bdd-dist}) and that $\cA_Q$ was chosen arbitrarily, we deduce $\Expect[\regretf_P(\vech, \vecD, \cH)] = \Omega(\sqrt{\lDimH \cdot T})$.
\end{proof}

\begin{theorem}[Upper bounds on expected loss/regret for quantum online binary classification] \label{thm:quantum-online-boolean-ub}
Let $\cH \subseteq \{0,1\}^\cX$, be a hypothesis class, and $h^\star \in \cH$. For every quantum adversary, there exists a quantum online learner for $\cH$ that satisfies
\begin{align*}
    \Expect[\lossf_P(\vech, \vecD, h^\star)] & = \bigO(\lDim(\cH)) \; \quad \text{(realizable), and} \\
    \Expect[\regretf_P(\vech, \vecD, \cH)] & = \bigO(\sqrt{\lDim(\cH) \cdot T}) \; \quad \text{(agnostic).}
\end{align*}
\end{theorem}
\begin{proof}
For a na\"ive algorithm that, at each round $t$, measures $\ket{\psi_t}$ in the standard basis and employs a classical learner to learn from the observed classical outputs, the desired upper bounds are guaranteed by Theorems \ref{thm:loss-upper-bdd-dist} and \ref{thm:regret-upper-bdd-dist}.
\end{proof}

\subsection{Multiclass Classification}

Here, we present bounds on the expected regret (in both the realizable and the agnostic cases) of a quantum online learner for a multiclass hypothesis class.

\begin{theorem}[Lower bounds on expected loss/regret for quantum online multiclass classification] \label{thm:quantum-online-multiclass-lb}
Let $\cH \subseteq \cY^\cX, \, \text{with} \; |\cY| = k > 2$, be a hypothesis class, and $h^\star \in \cH$. For every quantum  online learner of $\cH$, there exists a quantum adversary such that
\begin{align*}
    \Expect[\lossf_P(\vech, \vecD, h^\star)] & = \Omega(\mclDim(\cH)) \; \quad \text{(realizable), and} \\
    \Expect[\regretf_P(\vech, \vecD, \cH)] & = \Omega(\sqrt{\mclDim(\cH) \cdot T}) \; \quad \text{(agnostic).}
\end{align*}
\end{theorem}

\begin{proof}
The proof is identical to that of Theorem \ref{thm:quantum-online-boolean-lb}, where now, for the corresponding classical learners, $\Expect[\lossf_P(\vech_{\cA_C}, \vecD, h^\star)] = \Omega({\mclDim(\cH)})$ (from Theorem \ref{thm:loss-lower-bdd-dist-mc}), and $\Expect[\regretf_P(\vech_{\cA_C}, \vecD, \cH)] = \Omega(\sqrt{\mclDim(\cH) \cdot T})$ (from Theorem \ref{thm:regret-lower-bdd-dist-mc}).
\end{proof}

\begin{theorem}[Upper bounds on expected loss/regret for quantum online multiclass classification] \label{thm:quantum-online-multiclass-ub}
Let $\cH \subseteq \cY^\cX, \, \text{with} \; |\cY| = k > 2$, be a hypothesis class, and $h^\star \in \cH$. For every quantum adversary, there exists a quantum online learner for $\cH$ that satisfies
\begin{align*}
    \Expect[\lossf_P(\vech, \vecD, h^\star)] & = \bigO(\mclDim(\cH)) \; \quad \text{(realizable), and} \\
    \Expect[\regretf_P(\vech, \vecD, \cH)] & = \bigO(\sqrt{\mclDim(\cH) \cdot T\log(Tk)}) \; \quad \text{(agnostic).}
\end{align*}
\end{theorem}
\begin{proof}
Once again, we consider the measure-and-learn-classically quantum learner for which the desired upper bounds are now guaranteed by Theorems \ref{thm:loss-upper-bdd-dist-mc} and \ref{thm:regret-upper-bdd-dist-mc}.
\end{proof}

\subsection{Takeaways}\label{sec:takeaways}

Before we end this section on online learning with quantum examples, we note that the proofs for the expected regret upper bounds were established by a quantum online learner that performs a measurement and subsequently learns classically. The fact that the upper bounds thus obtained are \textit{identical} to the lower bounds, in all but one setting\footnote{The exception is the online multiclass agnostic case, where the quantum upper and lower bounds differ by a factor of $\sqrt{\log(Tk)}$.}, shows that the performance of this measure-and-learn-classically learner is as good as the best ``genuine'' quantum online learner in these settings. We feel that this is consistent with the overall message of this paper, viz., that there is limited power in quantum examples to speed up learning especially when the adversary is allowed to play arbitrary distributions (including very degenerate ones like point masses). 

Recently, \cite{hanneke2023multiclass} improved the classical upper bound for the online multiclass agnostic case to $\tilde{\bigO}(\sqrt{\mclDimH T})$\footnote{Here, $\tilde{\bigO}(\cdot)$ hides $\sqrt{\log\Big(\frac{T}{\mclDimH}\Big)}$ factors.}, which removes all $k$ dependence (cf. the $\sqrt{\log k}$ factor that appears in our corresponding \textit{quantum} upper bound in Theorem \ref{thm:quantum-online-multiclass-ub}). Meanwhile, we believe our analysis in the proof of Theorem \ref{thm:regret-upper-bdd-dist-mc} (which establishes the classical bound for the measure-and-learn-classically quantum learner in Theorem \ref{thm:quantum-online-multiclass-ub}) is tight, and so we suspect that any removal of the $k$-dependence in this setting would involve investigating into a ``genuine'' quantum online learning algorithm, which may involve a quantum-specific combinatorial parameter that characterizes learning. We identify this as an open question for future work.
\begin{itemize}
    \item What is the \textit{tight} expected regret bound for quantum online multiclass agnostic learning when the label space is unbounded (i.e., when the number of classes $k \to \infty$)?
\end{itemize}

\section{Conclusion}\label{sec:conclusion}

In this work, we partially resolved an open question of \cite{arunachalam2018optimal} by characterizing the sample complexity of multiclass learning (for $2 < k < \infty$). With recent work \cite{brukhim2022characterization} fully characterizing \textit{classical} multiclass learnability (including the case when $k \to \infty$) via the DS dimension, we ask whether \textit{quantum} multiclass learnability is also fully characterized by the DS dimension. We know that the upper bound in \cite{brukhim2022characterization} also holds in the quantum case by \textit{measure-and-learn-classically}. However, since the classical lower bound involving the DS dimension (Theorem 2 of \cite{daniely2014optimal}) uses \textit{transductive learning} which has no clear analog for a quantum example (ref. (\ref{example:quantum-realizable}) and (\ref{example:quantum-agnostic})), providing a quantum lower bound involving the DS dimension has proved to be non-trivial. We identify this as an open question for future work.
\begin{itemize}
    \item What is the \textit{tight} quantum sample complexity bound for batch multiclass learning, in both the realizable and agnostic settings,  when the label space is unbounded (i.e., when $k \to \infty$)?
\end{itemize}

In the batch setting, the sample complexity upper bounds were trivial to establish due to the quantum learner's ability to measure quantum examples and learn classically on the resulting output. In the online setting, the expected regret lower bounds, in turn, were trivial due to the adversary's ability to provide point masses $D_t$ at each $t$, rendering each quantum example equivalent to a classical example. This prompts us to ask the following question.
\begin{itemize}
    \item What happens when we impose restrictions on $D_t$ to force it away from a point mass? Would the expected regret bounds for the canonical classical online model (Section \ref{model:classical-adversary-provides-an-input}), classical adversary-provides-a-distribution model (Sections \ref{model:classical-adversary-provides-a-distribution} and \ref{model:classical-adversary-provides-a-distribution-agnostic}), and the quantum online model in Section \ref{sec:quantum-online-model} all diverge from one another?
\end{itemize}

\bibliographystyle{quantum}
\bibliography{refs}

\end{document}